\documentclass[twoside,11pt]{article}

\usepackage{blindtext}

\usepackage{graphicx} 

\usepackage[dvipsnames]{xcolor}
\colorlet{DarkGreen}{green!50!Black}
\usepackage{fullpage}
\usepackage[colorlinks]{hyperref}
\hypersetup{
    citecolor=DarkGreen,
    linkcolor=blue,
    filecolor=magenta,      
    urlcolor=cyan
    }
\usepackage{natbib}

\usepackage{amsmath,amsfonts,amssymb,amsthm}
\usepackage{mathtools}

\usepackage{subcaption}

\newcommand{\sign}{\mathrm{sign}}
\newcommand{\myspan}{\mathrm{span}}

\newcommand*{\eqdef}{\stackrel{\textup{def}}{=}}

\newcommand{\cN}{{\cal N}}

\newcommand{\R}{{\mathbb R}}

\newcommand{\tw}{\tilde{w}}

\newcommand{\diag}{\mathrm{diag}}

\newcommand{\E}{{\bf E}}

\renewcommand{\Pr}{{\bf Pr}}
\newcommand{\co}{\mathrm{co}}
\renewcommand{\Re}{{\R}}

\newcommand{\lv}{\lVert}
\newcommand{\rv}{\rVert}

\newtheorem{theorem}{Theorem}

\newtheorem{lemma}[theorem]{Lemma}

\newtheorem{definition}[theorem]{Definition}

\newtheorem{remark}[theorem]{Remark}

\DeclareMathOperator{\lambdamax}{\lambda_{\text max}}

\newcommand{\astrut}{\rule[-.3\baselineskip]{0pt}{\baselineskip}}

\usepackage{lastpage}

\begin{document}



\title{\bf The Dynamics of Sharpness-Aware Minimization: \\
    Bouncing Across Ravines \\ 
and Drifting Towards Wide Minima \\}

\author{Peter L. Bartlett\thanks{Also affiliated with University of California, Berkeley.},\hspace{2.5pt}  Philip M. Long and Olivier Bousquet \\
       Google \\
       1600 Amphitheatre Parkway \\
       Mountain View, CA 94040 \\
       $\{$peterbartlett,plong,obousquet$\}$\@google.com \\
       }
       

\date{}

\maketitle

\begin{abstract}
We consider Sharpness-Aware Minimization (SAM), a gradient-based optimization method for deep networks that has exhibited performance improvements on image and language prediction problems. 
We show that when SAM is applied with a
convex quadratic objective, for most random initializations it converges to a cycle that oscillates between either side of the minimum in the direction with the largest curvature, and we provide bounds on the rate of convergence.

In the non-quadratic case, we show that such oscillations effectively perform gradient descent, with a smaller step-size, on the spectral norm of the Hessian.
In such cases,
SAM's update 
may be regarded as
a third derivative---the derivative of the Hessian in the leading eigenvector direction---that encourages drift toward wider minima.
\end{abstract}


\section{Introduction}

The broad practical impact of deep learning has heightened interest in many of its surprising characteristics: simple gradient methods applied to deep neural networks seem to efficiently optimize nonconvex criteria, reliably giving a near-perfect fit to training data, but exhibiting good predictive accuracy nonetheless 
\citep[see][]{bmr-dlasp-21}.
Optimization methodology is widely believed to affect statistical performance by imposing some kind of implicit regularization, and there has been considerable effort devoted to understanding the behavior of optimization methods and the nature of solutions that they find. For instance,~\citet{barrett2020implicit} and \citet{smith2021origin} show that discrete-time gradient descent and stochastic gradient descent can be viewed as gradient flow methods applied to penalized losses that encourage smoothness, and~\citet{soudry2018implicit} amd \citet{azulay2021implicit} identify the implicit regularization imposed by gradient flow in specific examples, including linear networks.

We consider {\em Sharpness-Aware Minimization} (SAM), a recently introduced~\citep{foret2021sharpnessaware} gradient optimization method that has exhibited substantial improvements in prediction performance for deep networks applied to image classification~\citep{foret2021sharpnessaware} and NLP~\citep{bahri-etal-2022-sharpness} problems.

In introducing SAM, Foret {\it et al\/} motivate it 
using a
minimax optimization problem
  \begin{equation}\label{eqn:SAM-idealized}
    \min_w \max_{\|\epsilon\|\le\rho} \ell(w+\epsilon),
  \end{equation}
where $\ell:\Re^d\to\Re$ is an empirical loss defined on the parameter space $\Re^d$, $\|\cdot\|$ is the Euclidean norm on the parameter space, and $\rho$ is a scale parameter. By viewing the difference
  \[
    \max_{\|\epsilon\|\le\rho}
       \ell(w+\epsilon) - \ell(w)
  \]
as a measure of the sharpness of the empirical loss $\ell$ at the parameter value $w$, the criterion in~\eqref{eqn:SAM-idealized} allows a trade-off between the empirical loss and the sharpness,
    \[
      \max_{\|\epsilon\|\le\rho} \ell(w+\epsilon)
       =  \ell(w) + \underbrace{\max_{\|\epsilon\|\le\rho}
       \ell(w+\epsilon) - \ell(w)}_{\text{sharpness}}.
    \]
In practice, SAM works with a simplification based on gradient measurements, starting with an initial parameter vector $w_0\in\Re^d$ and updating the parameters at iteration $t$ via
    \begin{equation}\label{e:sam}
      w_{t+1} = w_t-\eta\nabla \ell\left(w_t + \rho\frac{\nabla
        \ell(w_t)}{\|\nabla \ell(w_t)\|}\right),
    \end{equation}
where $\eta$ is a step-size parameter. Our goal in this paper is to understand the nature of the solutions that the SAM updates~\eqref{e:sam} lead to.

In Sections~\ref{s:quad} and~\ref{s:proof_quadratic}, we consider SAM with a convex quadratic criterion. 
The key insight is that it is equivalent to a gradient descent method for a certain non-convex criterion whose stationary points correspond to oscillations around the minimum in the directions of the eigenvectors of the Hessian of the loss. The only stable stationary point corresponds to the leading eigenvector direction: `bouncing across the ravine'. (Notice that this is not the solution to the motivating minimax optimization problem~\eqref{eqn:SAM-idealized}, which is the minimum of the quadratic criterion.)

In Section~\ref{s:nonquad}, we consider SAM near a smooth minimum of the loss function $\ell$ with a positive semidefinite Hessian. For parameters corresponding to the solutions for the quadratic case, we see that the SAM updates can be decomposed into two components. There is a large component in the direction of the oscillation (bouncing across the ravine), and there is a smaller component in the orthogonal subspace that corresponds to descending the gradient of the spectral norm of the Hessian. Thus, SAM is able to drift towards wide minima by exploiting a specific third derivative (the gradient of the second derivative in the leading eigenvalue direction) with only two gradient computations per iteration.
In Section~\ref{s:conclusions}, we present some open problems, the most important of which is elucidating the relationship between wide minima of empirical loss and statistical performance.

\section{Additional Related Work}

\citet{du2022efficient} 
proposed a more computationally efficient variant of SAM.
\citet{beugnot2022benefits} studied the
effect of a large learning rate with early stopping on spectrum of the
Hessian in the case of quadratic loss.

\citet{cohen2020gradient} provided a variety of natural
settings where, empirically, when neural networks are trained with batch gradient
descent and a fixed learning rate $\eta$, the spectral norm of the Hessian
tends toward $2/\eta$, the ``edge of stability''.  Here, if the
gradient is aligned with the principal direction of the
Hessian, the solution ``bounces across the ravine'', as in the
analysis of this paper.  A number of theoretical treatments
of this phenomenon have since been proposed
\citep{pmlr-v162-ahn22a,arora2022understanding,damian2022self}.
The most closely related of those to this paper is
the work of \citet{damian2022self},
who also described conditions under which ``bouncing across the
ravine'' tends to decrease the spectral norm of the Hessian.

In independent work posted to arXiv
after the initial version of this paper, \citet{wen2022does}
performed a variety of analyses of SAM and some related algorithms.  Their results included 
showing that SAM almost surely converges in the limit
in the convex quadratic case,
along with asymptotic analysis showing that,
once SAM gets close enough to the manifold
of loss minimizers,
it approximately tracks the path on a loss-minimizing manifold of gradient flow with respect to
the spectral norm of the Hessian, under smoothness assumptions
on the loss.
They also showed that the stochastic version of SAM, in which both gradients at each step are estimated from a single training example, approximately tracks the path of gradient flow with respect to the trace of the Hessian.

%

\section{SAM with Quadratic Loss: Bouncing Across Ravines}
\label{s:quad}

We first consider the application of SAM to minimize a convex quadratic
objective $\ell$.  Without loss of generality, we assume that the minimum of $\ell$
is at zero, the eigenvectors of $\ell$'s Hessian are the coordinate axes, 
and the eigenvalues
are sorted by the indices of the eigenvectors.  Accordingly,
for $\Lambda=\diag(\lambda_1,\ldots,\lambda_d)$ with $\lambda_1\ge\cdots\ge\lambda_d > 0$, we consider loss $\ell(w)=\frac{1}{2}w^\top\Lambda w$. Then $\nabla\ell(w) = \Lambda w$ and SAM sets
  \begin{align}
      w_{t+1}
      & = w_t - \eta\nabla\ell\left(w_t+\rho\frac{\nabla\ell(w_t)}{\|\nabla\ell(w_t)\|}\right) \notag\\
      &= \left(I - \eta\Lambda - \frac{\eta \rho}{\|\Lambda w_t\|}\Lambda^2\right)w_t. \label{e:SAM-quadratic}
  \end{align}

The following is our main result.
%
%

\begin{theorem}
\label{t:medium}
There are polynomials $p$ and $p'$ and an absolute constant $c$ such that
the following holds.
For any eigenvalues $\lambda_1 > \lambda_2 \geq ... \geq \lambda_d > 0$, 
loss $\ell(w)=\frac{1}{2}w^\top\Lambda w$ with $\Lambda=\diag(\lambda_1,\ldots,\lambda_d)$,
any  neighborhood size $\rho > 0$,
any step size 
$0 < \eta < \frac{1}{2 \lambda_1}$,
and any $\delta > 0$,
if $w_0$ is sampled from a continuous probability
distribution over $\R^d$ 
\begin{itemize}
    \item whose density is bounded
above by $A \in \R$, and
    \item for $R > \eta \rho \lambda_1$ and $q>0$, with probability at least $1-\delta$, $\|w_0\|\le R$ and 
      $w_{0,1}^2 \ge q$,
\end{itemize}
and $w_1, w_2,...$ are obtained through the
SAM update \eqref{e:sam},
then, if $\kappa = \lambda_1/\lambda_d$, for all 
\[
\epsilon < p'(1/\lambda_1,\lambda_d, \eta,\rho,\delta,1/\rho,A,R),
\]
with probability
$1 - 2 \delta$, for all
\begin{align*}
& t \geq  
 \left( \frac{\kappa^5}
     {\eta \lambda_d \min\left\{\eta\lambda_d,
    \lambda_1^2/\lambda_2^2-1\right\}}
    + d
    \right)
    p\left(\log\left( \frac{1}{\epsilon} \right)
          \right)
\end{align*}
one of the following holds:
\begin{itemize}
    \item $\lv w_t - \frac{\eta \rho \lambda_1 e_1}{2-\eta\lambda_1} \rv \leq \epsilon$
and $\lv w_{t+1} + \frac{\eta \rho \lambda_1 e_1}{2-\eta\lambda_1} \rv \leq \epsilon$, or
    \item $\lv w_t + \frac{\eta \rho \lambda_1 e_1}{2-\eta\lambda_1} \rv \leq \epsilon$
and $\lv w_{t+1} - \frac{\eta \rho \lambda_1 e_1}{2-\eta\lambda_1} \rv \leq \epsilon$.
\end{itemize}
\end{theorem}

Theorem~\ref{t:medium} has the following corollary.
\begin{theorem}
\label{t:simple}
For any eigenvalues $\lambda_1 > \lambda_2 \geq  ... \geq \lambda_d > 0$, any 
neighborhood size $\rho > 0$, and
any step size 
$0 < \eta < \frac{1}{2 \lambda_1}$,
if $w_0$ is sampled from a continuous probability
distribution over $\R^d$ with $\E[\| w_0 \|^2] < \infty$,
then, almost surely, for all $\epsilon > 0$, for all
large enough $t$, the iterates of SAM applied to the quadratic loss
$\ell(w)=\frac{1}{2}w^\top\diag(\lambda_1,\ldots,\lambda_d) w$ satisfy:
\begin{itemize}
    \item $\lv w_t - \frac{\eta \rho \lambda_1 e_1}{2-\eta\lambda_1} \rv \leq \epsilon$
and $\lv w_{t+1} + \frac{\eta \rho \lambda_1 e_1}{2-\eta\lambda_1} \rv \leq \epsilon$, or
    \item $\lv w_t + \frac{\eta \rho \lambda_1 e_1}{2-\eta\lambda_1} \rv \leq \epsilon$
and $\lv w_{t+1} - \frac{\eta \rho \lambda_1 e_1}{2-\eta\lambda_1} \rv \leq \epsilon$.
\end{itemize}
\end{theorem}

Our analysis shows that, when SAM is initialized far from the optimum, training
proceeds in two stages.  Early, the objective function is reduced exponentially
fast, with the most rapid progress made in the directions with highest variance.
This can be seen, for example, in Figure~\ref{f:sam_quad_all},
which plots the first 30 iterates of SAM initialized at $(2,2)$ in the
case that $\lambda_1 = 1$ and $\lambda_2 = 1/2$, $\eta = 1/5$ and $\rho = 1$.
\begin{figure}
    \centering
    \begin{subfigure}{2.8in}
              \includegraphics[width=2.7in]{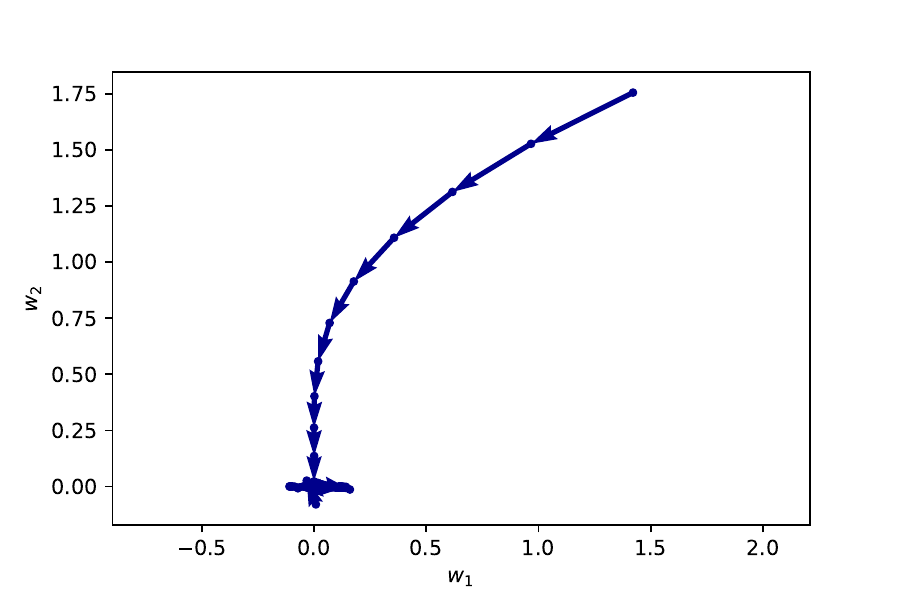}
              \caption{All the iterates}
              \label{f:sam_quad_all}
    \end{subfigure}
    \begin{subfigure}{2.8in}
    \includegraphics[width=2.7in]{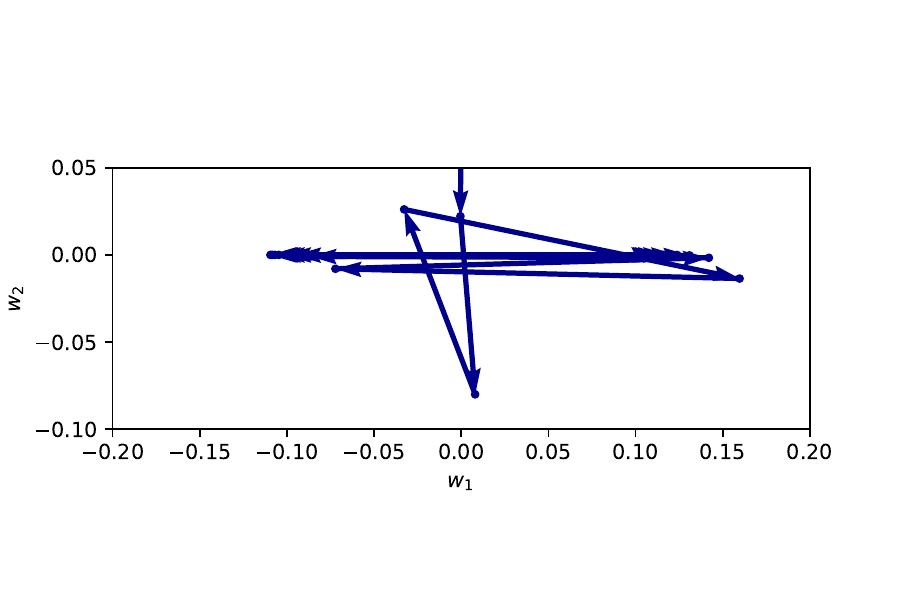}
      \caption{The iterates close to the origin}
      \label{f:sam_quad_close}
      \end{subfigure}
    \caption{The first 30 iterates of SAM, initialized at $(2,2)$ with $\lambda_1 = 1$ and $\lambda_2 = 1/2$, $\eta = 1/5$ and $\rho = 1$.}
    \label{f:sam_quad}
\end{figure}
After a certain point, however, SAM's iterates ``overshoot'' in the direction
of highest variance, as can be seen in Figure~\ref{f:sam_quad_close}, which
is the same as Figure~\ref{f:sam_quad_all}, except zoomed in to the region near
the origin, where the details of the later iterates can be seen.
During this second phase, the share of the length of the parameter vector
in the first component increases, and the process converges to the oscillation
described in Theorem~\ref{t:simple}.  Note that, as illustrated in
Figure~\ref{f:sam_quad_all}, due to the normalization by
$\| w_t\|$, the parameter vector can jump away from a position very close
to the origin, with a correspondingly very small loss.
However, as we will see, the training process makes steady progress with respect
to a potential function that we will define in Section~\ref{s:gd}.

\section{Proof of Theorem~\ref{t:medium}}
\label{s:proof_quadratic}

In this section, we prove the following theorem, which implies Theorem~\ref{t:medium}.
We denote $\max \{ z, 0\}$ by $[z]_+$.
\begin{theorem}
\label{t:quadratic}
There is an absolute constant $c$ such that,
for any eigenvalues $\lambda_1 > \lambda_2 \geq ... \geq \lambda_d > 0$,
loss $\ell(w)=\frac{1}{2}w^\top\Lambda w$ with $\Lambda=\diag(\lambda_1,\ldots,\lambda_d)$, any
neighborhood size $\rho > 0$,
any initialization parameters $R, A, q > 0$,
and any step size 
$0 < \eta < \frac{1}{2 \lambda_1}$,
for all $0 < \epsilon < \min \left\{ \sqrt{\eta \lambda_1/2}, 1/(2 \rho\lambda_1), \eta \rho\lambda_1^2/2\right\}$,
for all $\delta > 0$,
if $w_0$ is sampled from a continuous probability
distribution over $\R^d$ 
\begin{itemize}
    \item whose density is bounded
above by $A \in \R$, and
    \item with probability at least $1-\delta$, $\|w_0\|\le R$ and 
      $w_{0,1}^2 \ge q$,
\end{itemize}
and $w_1, w_2,...$ are obtained through the
SAM update~\eqref{e:sam},
then, 
with probability
$1 - 2 \delta$, for all
\begin{align*}
& t \geq  \frac{6 \lambda_1^5}
     {\eta \lambda_d^6 \min\left\{\eta\lambda_d,
    \frac{\lambda_1^2}{\lambda_2^2}-1\right\}}   \log \left(
                    \frac{4 }
                         { \eta \lambda_1}
                   \right) \\
& \hspace{0.2in}
         +\frac{1}{\min\left\{\eta\lambda_d,
    \frac{\lambda_1^2}{\lambda_2^2}-1\right\}} 
           \left(
         \log\left(  \frac{4 (1 + \eta \rho \lambda_1^2)^2}
                       {\lambda_d^2 \epsilon^2}\right)
          +
    \log\left(\frac{R^2}{q}\right)
        \right)
        \\
& \hspace{0.2in}
         +
       \frac{2 \left[\log\left(\frac{R}{ \eta \rho \lambda_1}\right)\right]_+}{\eta\lambda_d \min\left\{\eta\lambda_d,
    \frac{\lambda_1^2}{\lambda_2^2}-1\right\}}
                              \Bigg(
                              \log\left( 2 \lambda_1 R \right)
                              + 
                                \frac{\left[ \log\left(\frac{ R}{\eta \rho \lambda_1}\right) \right]_+ \log\left(
                                  \frac{9\cdot 6^{d+3} R^3}
                                     {(\eta\lambda_d)^{d+3}(\eta \rho \lambda_1)^3}\right)}{\eta\lambda_d}  \\
& \hspace{2.2in} + \log\left(\frac{4\pi^{d/2}(4 \eta \rho \lambda_1^2)^{d-1}\left[\log\left(\frac{ R}{\eta \rho \lambda_1}\right)\right]_+ A}{\Gamma(d/2)\delta \eta\lambda_d}\right)      
                              \Bigg)
        \\
& \hspace{0.2in}
     + \frac{6}{\eta \lambda_1} \ln \left( \frac{ 2 (1 + \eta \rho \lambda_1^2) }{\lambda_d \epsilon} \right)
\end{align*}
one of the following holds:
\begin{itemize}
    \item $\lv w_t - \frac{\eta \rho \lambda_1 e_1}{2-\eta\lambda_1} \rv \leq \epsilon$
and $\lv w_{t+1} + \frac{\eta \rho \lambda_1 e_1}{2-\eta\lambda_1} \rv \leq \epsilon$, or
    \item $\lv w_t + \frac{\eta \rho \lambda_1 e_1}{2-\eta\lambda_1} \rv \leq \epsilon$
and $\lv w_{t+1} - \frac{\eta \rho \lambda_1 e_1}{2-\eta\lambda_1} \rv \leq \epsilon$.
\end{itemize}
\end{theorem}

The proof of Theorem~\ref{t:quadratic} requires some lemmas, which we prove first.
Throughout this section, we assume that $\eta\lambda_1<1/2$ and we highlight where the assumption $\lambda_1>\lambda_2$ is used.

The evolution of the gradient $\nabla \ell(w_t) =\Lambda w_t$ plays
a key role in the dynamics of SAM.
To simplify expressions, we refer to it using the shorthand $v_t$. 
Substituting into the SAM update~\eqref{e:SAM-quadratic} for the quadratic loss gives
  \begin{align*}
      v_{t+1}
      & = \left(I - \eta\Lambda - \frac{\eta \rho}{\|v_t\|}\Lambda^2\right)v_t,
 \end{align*}
so, for all $i \in [d]$ and all $t$, we have 
\begin{align*}
 v_{t+1,i}
   & = \left(
  1 - \eta\lambda_i  - \frac{\eta  \rho \lambda_i^2}{\|v_t\|}\right)
       v_{t,i} \\
   & = (1 - \eta\lambda_i) \left(
   \|v_t\| - \frac{\eta \rho \lambda_i^2}{1 - \eta\lambda_i }\right)
      \frac{v_{t,i}}{\|v_t\|} \\
      &= (1-\eta\lambda_i)\left(\|v_t\|-\gamma_i\right)\frac{v_{t,i}}{\|v_t\|}, 
  \end{align*}
where $\gamma_i:= \eta \rho \lambda_i^2/(1-\eta\lambda_i)$.

We need the following technical lemma.
\begin{lemma}\label{lemma:compare-quadratics}
For $x> 0$,
$0\le a < b$, $\alpha>\beta \geq 0$, and $a\alpha\ge b\beta$, we have
$a^2(x-\alpha)^2>b^2(x-\beta)^2$ iff $x<(a\alpha+b\beta)/(a+b)$.
\end{lemma}
\begin{proof}
Substituting shows that $b^2(x-\beta)^2-a^2(x-\alpha)^2=0$ at $x=(a\alpha+b\beta)/(a+b)\ge 0$. Also, $b^2(x-\beta)^2-a^2(x-\alpha)^2\le 0$ at $x=0$, which shows that the other zero of this convex quadratic occurs at $x\le 0$.
\end{proof}

\subsection{Some properties}
The following lemma identifies some properties of SAM with the convex quadratic criterion. It shows that the magnitudes of the components of the gradient vector $v_t$ have fixed points under SAM's update when the gradients are in the eigenvector directions and at distance $\beta_i$ from the minimum, it shows that the norm of $v_t$ determines how the magnitudes of its components grow, both in absolute terms (where the critical values are the $\beta_i$) and 
relative to the first component 
(where the critical values are the $\alpha_i$), and it shows that, for $b = \eta \rho \lambda_1^2$, the set 
$\{v : \|v\|\le b\}$ 
is absorbing.
Recall that we have assumed that $\eta\lambda_1<1/2$.
\begin{lemma}
\label{l:one_step}
For $i=1,\ldots, d$, define
  \begin{align*}
      \gamma_i &= \frac{\eta \rho\lambda_i^2}{1-\eta\lambda_i}, \\
      \beta_i  &=\frac{1-\eta\lambda_i}{2-\eta\lambda_i}\gamma_i
                =\frac{\eta \rho\lambda_i^2}{2-\eta\lambda_i}
               ,\\
      \alpha_i &=\frac{(1-\eta\lambda_1)\gamma_1+(1-\eta\lambda_i)\gamma_i}
        {1-\eta\lambda_1+1-\eta\lambda_i} \\
    b & = (1-\eta\lambda_1)\gamma_1 = \eta \rho \lambda_1^2. \\ 
  \end{align*}
We have
  \begin{enumerate}
      \item\label{part:one} 
      $
        \left\{\left(s_1,\ldots,s_d\right):
        \exists v_t, \;
        \forall
        \text{$1\le i\le d$, }v_{t+1,i}^2=v_{t,i}^2=s_i\right\}
        =
        \{ 0 \} \cup \bigcup_{i = 1}^d \co\{\beta_i^2e_j:\beta_j=\beta_i\}
      $,\\
      where $\co(S)$ denotes the convex hull of a set $S$ and $e_j$ is the $j$th basis vector in $\Re^d$.
      \item\label{part:two} For $1\le i\le d$, $v_{t+1,i}^2 < v_{t,i}^2$ iff $\|v_t\|>\beta_i$.
      \item\label{part:four} Suppose $\lambda_1>\lambda_2$. Then for $i\in\{2,\ldots,d\}$,
      $\frac{v_{t+1,1}^2}{v_{t+1,i}^2}
            >\frac{v_{t,1}^2}{v_{t,i}^2}$
      iff $\|v_t\|<\alpha_i$.
      \item\label{part:five} $\beta_d\le\cdots\le\beta_1\le\alpha_d\le\cdots\alpha_2\le\alpha_1=\gamma_1$ and $\beta_1\le b$. Furthermore, if $\lambda_d>0$ then $\beta_1<\alpha_d$.
      \item\label{part:three}
      $\|v_t\|\le b$ implies $\|v_{t+1}\|\le b$.
      %
  \end{enumerate}
\end{lemma}

\begin{proof}
We have
  \[
    v_{t+1,i}^2=(1-\eta\lambda_i)^2\left(\|v_t\|-\gamma_i\right)^2\frac{v_{t,i}^2}{\|v_t\|^2},
  \]
and so, for all $i$, if $v_{t+1,i}^2=v_{t,i}^2$, then, either 
$v_{t,i}^2=0$ 
or $(\|v_t\|-\gamma_i)^2=\|v_t\|^2/(1-\eta\lambda_i)^2$. This quadratic equation has only one non-negative solution, 
$\|v_t\|=\beta_i$ 
(and for 
$\|v_t\|>\beta_i$
$v_{t+1,i}^2<v_{t,i}^2$, proving part~\ref{part:two}). And so if $v_{t,i}\not=0$ for some $i$, then every $v_{t,j}$ with $\beta_j\not=\beta_i$ must have $v_{t,j}=0$, and in that case, $\|v_t\|^2=\sum_{j:\beta_j=\beta_i} v_{t,j}^2$. This proves part~\ref{part:one}.

To see why part~\ref{part:three} is true, notice that,
if $\| v_t \| \leq b$, we have
  \begin{align*}
      \|v_{t+1}\|^2
        & = \sum_i (1-\eta\lambda_i)^2(\|v_t\|-\gamma_i)^2 \frac{v_{t,i}^2}{\|v_t\|^2} \\
        & \leq \sum_i (1-\eta\lambda_i)^2\max\{\|v_t\|^2,\gamma_i^2\} \frac{v_{t,i}^2}{\|v_t\|^2} \\
        & \le \max_i \left\{(1-\eta\lambda_i)^2 \max\{\|v_t\|^2,\gamma_i^2\}\right\} \\
        & \leq \max_i \left\{\max\{(1-\eta\lambda_i)^2 b^2,(1-\eta\lambda_i)^2 \gamma_i^2\}\right\} \\
        & = \max_i \left\{\max\{(1-\eta\lambda_i)^2 (1-\eta\lambda_1)^2\gamma_1^2,(1-\eta\lambda_i)^2 \gamma_i^2\}\right\} \\
        & = \max_i \left\{(1-\eta\lambda_i)^2 \gamma_i^2\right\} \\
        & = (1-\eta\lambda_1)^2\gamma_1^2 = b^2.
  \end{align*}

For part~\ref{part:four}, $v_{t+1,1}^2/v_{t,1}^2 > v_{t+1,i}^2/v_{t,i}^2$ iff
  \[
    (1-\eta\lambda_1)^2(\|v_t\|-\gamma_1)^2 > (1-\eta\lambda_i)^2(\|v_t\|-\gamma_i)^2.
  \]
But because
$1-\eta\lambda_1<1-\eta\lambda_2\le 1-\eta\lambda_i$, $\gamma_1>\gamma_2\ge\gamma_i$, and $(1-\eta\lambda_1)\gamma_1>(1-\eta\lambda_2)\gamma_2\ge (1-\eta\lambda_i)\gamma_i$, we can apply Lemma~\ref{lemma:compare-quadratics}, which shows that for $\|v_t\|>0$, this is equivalent to $\|v_t\|<\alpha_i$. Clearly, the inequality also holds at $\|v_t\|=0$.

To see part~\ref{part:five}, first notice that,
for $f(x) = x^2/(2 - \eta x)$, $f'(x) \geq 0$ for
all $x \in [0,1/\eta)$, which implies
that the $\beta_i$ are non-increasing in $i$.
Also,
  \[
    \alpha_d = \frac{(1-\eta\lambda_1)\gamma_1 + (1-\eta\lambda_d)\gamma_d}
    {1-\eta\lambda_1+1-\eta\lambda_d} \ge \frac{1-\eta\lambda_1}
    {1-\eta\lambda_1+1-\eta\lambda_d}\gamma_1 \ge \frac{1-\eta\lambda_1}
    {2-\eta\lambda_1}\gamma_1 = \beta_1,
  \]
and the last inequality is strict iff $\lambda_d>0$.
Also, for $i<j$, $\gamma_i\ge\gamma_j$ and $1-\eta\lambda_i\le 1-\eta\lambda_j$, hence
  \[
    \alpha_i = \frac{(1-\eta\lambda_1)\gamma_1 + (1-\eta\lambda_i)\gamma_i}
    {1-\eta\lambda_1+1-\eta\lambda_i} \ge \frac{(1-\eta\lambda_1)\gamma_1 + (1-\eta\lambda_i)\gamma_j}
    {1-\eta\lambda_1+1-\eta\lambda_i} \ge \frac{(1-\eta\lambda_1)\gamma_1 + (1-\eta\lambda_j)\gamma_j}
    {1-\eta\lambda_1+1-\eta\lambda_j} = \alpha_j.
  \]
Finally, $\beta_1<(1-\eta\lambda_1)\gamma_1$ because $\eta\lambda_1<1$.

%
\end{proof}

\subsection{Early descent}
\label{s:early}

In this section, we show that SAM rapidly descends towards the gradient ball $\|v_t\|\le b$ and that under our conditions on the initialization, it reaches this ball with the magnitude of the first component, $|v_{t,1}|$, bounded away from zero.
We shall see in Section~\ref{s:avoiding} that this ensures SAM,
applied to the quadratic loss
$\ell$,
converges to the leading eigenvector direction.

The following lemma shows that when the solution is far from the optimum, SAM
rapidly descends toward the optimum and the relative magnitude of the first component of the gradient does not get too small.

\begin{lemma}\label{l:fast_to_b}
Suppose that, for $R > 0$,
we have
$\|v_0\|\le R$. For any $t\ge T: =  [\log(R/ b)]_+/(\eta\lambda_d)$, we have $\|v_{t}\|\le b$.

Furthermore, if, for $\Delta > 0$, $\left|\|v_t\|-\gamma_1\right|\ge\Delta$ for all $t$, then there is a
$T_0\le T$ satisfying $\|v_{T_0}\|\le b$ and
\begin{align}
  \label{e:p_one_big}
    \frac{\sum_{i=2}^d v_{T_0,i}^2}{v_{T_0,1}^2} \le
    \left(\frac{2 R}{\Delta}\right)^{2T}
    \frac{\sum_{i=2}^d v_{0,i}^2}{v_{0,1}^2}.
\end{align}
Thus,
  \[
    \log\left(\frac{\sum_{i=2}^d v_{T_0,i}^2}{v_{T_0,1}^2}\right) \le
    \frac{2}{\eta\lambda_d}\left[\log\left(\frac{R}{ b}\right) \right]_+ \log\left(\frac{ 2 R}{\Delta}\right) +
    \log\left(\frac{R^2}{v_{0,1}^2}\right).
  \]
\end{lemma}

\begin{proof}
If $R \leq b$,
the lemma is an obvious consequence of Part 5 of Lemma~\ref{l:one_step} ;  assume for
the rest of the proof that $R > b$.

Notice that $\|v_t\|^2\ge(\|v_t\|-\gamma_i)^2$ if and only if $2\|v_t\|\ge\gamma_i$. But
  \[
     b=(1-\eta\lambda_1)\gamma_1
    \ge \gamma_1/2\ge\gamma_i/2.
  \]
Thus, for any $\|v_t\|\ge b$, we have
  \begin{align}
  \nonumber
      \|v_{t+1}\|^2
        & = \sum_{i=1}^d (1-\eta\lambda_i)^2\left(\|v_t\|-\gamma_i\right)^2\frac{v_{t,i}^2}{\|v_t\|^2} \\
          \nonumber
        & \le\max_i (1-\eta\lambda_i)^2\|v_t\|^2 \\
        \label{e:v_smaller}
        & = (1-\eta\lambda_d)^2\|v_t\|^2.
  \end{align}
From Lemma~\ref{l:one_step}, part~\ref{part:three}, if $\|v_t\|\le b$ then $\|v_{t'}\|\le b$ for $t'\ge t$. Thus, for all $t$ satisfying $(1-\eta\lambda_d)^t\|v_0\|\le b$, we have $\|v_t\|\le b$. This is equivalent to
  \[
    t\ge\frac{\log(\|v_0\|/ b)}{\log(1/(1-\eta\lambda_d))}.
  \]
Since $\log(1-\eta\lambda_d)\le-\eta\lambda_d$,
it suffices if
  \[
    t\ge T = \frac{\left[\log(R/ b)\right]_+}{\eta\lambda_d}.
  \]

For the second part of the lemma, as long as $\|v_t\|\ge b$ we have
  \begin{align*}
      \frac{v_{t+1,i}^2}{v_{t+1,1}^2}
        & = \frac{(1-\eta\lambda_i)^2(\|v_t\|-\gamma_i)^2 v_{t,i}^2}
        {(1-\eta\lambda_1)^2(\|v_t\|-\gamma_1)^2 v_{t,1}^2} \\
        & \le \frac{(1-\eta\lambda_i)^2\|v_t\|^2 v_{t,i}^2}
        {(1-\eta\lambda_1)^2\Delta^2 v_{t,1}^2} \\
        & \le \frac{(1-\eta\lambda_d)^2R^2 v_{t,i}^2}
        {(1-\eta\lambda_1)^2\Delta^2 v_{t,1}^2}.
  \end{align*}
Thus, if $T_0$ is the first iterate for which
$\| v_{T_0} \| < b$, we have
\[
\frac{v_{T_0,i}^2}{v_{T_0,1}^2}
\leq  \left(\frac{2 R}{\Delta}\right)^{2T_0}
    \frac{\sum_{i=2}^d v_{0,i}^2}{v_{0,1}^2}
\leq  \left(\frac{2 R}{\Delta}\right)^{2T}
    \frac{\sum_{i=2}^d v_{0,i}^2}{v_{0,1}^2},
\]
completing the proof.
\end{proof}

If $\|v_t\| = \gamma_1$, then $v_{t',1} = 0$ for all $t' > t$, and, if $\|v_t\|$ is very close to $\gamma_1$,
the first component of $v_{t+1}$ could be
small enough that it takes a long time to recover.
Lemma~\ref{l:far_from_gammaone} establishes that this is unlikely.

\begin{lemma}\label{l:far_from_gammaone}
Fix a constant $A$
and $R > 0$.
For all $\delta > 0$, if $v_0 \in\Re^d$ is
chosen randomly from a distribution 
such that $\Pr[\| v_0 \| > R] \leq \delta$,
and whose density is bounded above by $A$, then, 
with probability $1 - 2 \delta$,
for all $t$, 
$|\|v_t\|-\gamma_1| \geq \Delta$, where
  \[
      \Delta = \frac{\Gamma(d/2)\delta}{4\pi^{d/2}(2\gamma_1)^{d-1}T_0 A}\left(\frac{(\eta\lambda_d)^{d+3}\gamma_1^3}{9\cdot 6^{d+3}R^3}\right)^{T_0}
  \]
and $T_0 \leq \frac{[\log(R/ b)]_+}{\eta\lambda_d}$.
Thus,
\[
  \log\frac{1}{\Delta} 
  \le 
  \frac{[\log(R/ b)]_+ \log\left(\frac{9\cdot 6^{d+3}R^3}{(\eta\lambda_d)^{d+3}\gamma_1^3}\right)}{\eta\lambda_d}  + \log\left(\frac{4\pi^{d/2}(2\gamma_1)^{d-1}[\log(R/ b)]_+ A}{\Gamma(d/2)\delta \eta\lambda_d}\right).
\]
\end{lemma}
\begin{proof}
Before delving into the details, here is the outline of the proof.
We argue that at every step when $\|v_t\|$ is bigger than $\gamma_1$, the density is small, and hence $\| v_{t+1} \|$ is unlikely to fall in the interval $[\gamma_1-\Delta,\gamma_1+\Delta]$. We consider all steps until $\|v_t\|<\gamma_1+\epsilon$, where $\epsilon$ is larger than $\Delta$ and is chosen so that, if $\|v_t\|<\gamma_1+\epsilon$, then
$\|v_{t+1}\|$ drops below $\gamma_1-\epsilon \leq \gamma_1-\Delta$.
We choose $\epsilon=\eta\lambda_d\gamma_1/(2-\eta\lambda_d)$ for this purpose:
the proof of the previous lemma shows that $\|v_{t+1}\|\le(1-\eta\lambda_d)\|v_t\|$, and
with our choice of $\epsilon$, $\|v_t\|<\gamma_1+\epsilon$ implies $\|v_{t+1}\|<(1-\eta\lambda_d)(\gamma_1+\epsilon)<\gamma_1-\epsilon$.
We compute an upper bound on the factor by which the density increases at each step when $\|v_t\|\ge\gamma_1+\epsilon$.
Lemma~\ref{l:fast_to_b} shows that there cannot be many such steps.

Let $f$ denote the mapping from $v_t$ to $v_{t+1}$ whose domain
is $\{ v : \| v \| \geq \gamma_1+\epsilon \}$, so that
if
we define $G = \diag(1-\eta\lambda_1,\ldots,1-\eta\lambda_d)$
and $H = \diag((1-\eta\lambda_1)\gamma_1,\ldots,(1-\eta\lambda_d)\gamma_d)$, then we can write
\[
f(v) = G v + H \frac{v}{\|v\|}.
\]

If $\mu_t$ is the density of $v_t$ and $f$ is invertible, and we denote the
Jacobian of $f^{-1}$ by
$\nabla f^{-1}$, then the density $\mu_{t+1}$ of $v_{t+1}$ is the pushforward measure
  \[
    \mu_{t+1}(x) = \mu_t(f^{-1}(x))\left|(\nabla f^{-1})(x)\right|.
  \]
To see that $f$ is indeed invertible, we write $v=r\hat v$, for $r>0$ and $\|\hat v\|=1$, and $y=f(v)$. Then
  \[
    y = G v -H \hat v = \left(rG-H \right)\hat v.
  \]
To see that $rG - H$ is invertible, note that
\begin{align*}
rG-H 
 & = \diag\left((1-\eta\lambda_1) r-(1-\eta\lambda_1)\gamma_1,
           ..., (1-\eta\lambda_d) r-(1-\eta\lambda_d)\gamma_d
      \right) \\
 & = \diag\left((1-\eta\lambda_1) (r - \gamma_1),
           ..., (1-\eta\lambda_d) (r - \gamma_d)
      \right),
\end{align*}
and each entry is positive because $r=\|v_t\|>\gamma_1\ge\gamma_i$.
This means $r$ is the unique solution to $y^\top(rG-H)^{-2}y=1$ and
$\hat{v} = \left(rG-H \right)^{-1} y$, and then 
\begin{align*}
v & = r \hat{v} 
  = r \left(rG-H \right)^{-1} y 
  = \left(G-H/r \right)^{-1} y. 
\end{align*}
To compute the Jacobian of $f^{-1}$, let's compute $dr/dy_j$ by differentiating the equation defining $r$. Adopting the
shorthand $g_i = G_{ii}$ and $h_i = H_{ii}$, we have
  \begin{align*}
    &&&\sum_{i=1}^d \frac{d}{dy_j} \frac{y_i^2}{(r g_i-h_i)^2} = 0 \\
    &\Leftrightarrow&
    &\frac{2y_j}{(r g_j-h_j)^2} - \frac{dr}{dy_j}\sum_{i=1}^d\frac{y_i^2 g_i}
    {2(r g_i-h_i)^3} = 0 \\
    &\Leftrightarrow&
    &\frac{dr}{dy_j}= \frac{2y_j}
    {(r g_j-h_j)^2\sum_{i=1}^d\frac{y_i^2 g_i}{2(r g_i-h_i)^3}}.
  \end{align*}
Next, we use $v_i=y_i/(g_i-h_i/r)$ to obtain the $i,j$ entry of the Jacobian of $f^{-1}$:
  \begin{align*}
      \frac{dv_i}{dy_j}
      & = \frac{\delta_{ij}}{g_i-h_i/r}
      + \frac{v_i h_i}{(g_i-h_i/r)^2r^2}\frac{dr}{dy_j} \\
      & = \frac{\delta_{ij}}{g_i-h_i/r}
      + \frac{2y_i h_i y_j}{(g_i-h_i/r)^2r^2(r g_j-h_j)^2\sum_{k=1}^d\frac{v_k^2g_k}
    {2(rg_k-h_k)^3}}.
  \end{align*}
Assembling these partial derivatives into the
Jacobian $\nabla f^{-1}$
yields the sum of an invertible diagonal matrix 
and a rank one matrix.  
We can use the fact that
$\det(A+ab^\top)=\det(A)(1+b^\top A^{-1}a)$ to get an explicit expression:
 \begin{align}
 \label{e:det}
     \det\left(\nabla f^{-1} \right)
            & = \frac{1+\displaystyle\sum_{i=1}^d \displaystyle\frac{2y_i^2h_i}{(g_i-h_i/r) r^2(r g_i - h_i)^2\sum_{k=1}^d\frac{y_k^2 g_k}{2(r g_k-h_k)^3}}}
      {\prod_{i=1}^d(g_i-h_i/r)}.
  \end{align}
Since $r\ge\gamma_1+\epsilon$, $r\le\|v_0\|$ and with probability at least $1-\delta$, $\|v_0\|\le R$, we have 
\begin{align*}
| g_i-h_i/r| &  = (1-\eta\lambda_i)\left(1 - \frac{\gamma_i}{r} \right) \\  
          & \geq (1-\eta\lambda_i)\left(1 - \frac{\gamma_i}{\gamma_1 + \epsilon} \right) \\ 
          & \geq (1-\eta\lambda_1)\left(1 - \frac{\gamma_1}{\gamma_1 + \epsilon} \right) \\ 
          & = (1-\eta\lambda_1)\left(1 - \frac{1}{1 + \epsilon/\gamma_1} \right) \\ 
          & = (1-\eta\lambda_1)\left(1 - \frac{1}{1 + \frac{\eta \lambda_d}{2 - \eta \lambda_d}}  \right).
\end{align*}
Recalling that $\eta \lambda_d < \eta \lambda_1 \leq 1/2$, this gives
\begin{align*}
| g_i-h_i/r| & \geq \frac{\eta \lambda_d}{6}.
\end{align*}
Defining $B = \frac{\eta \lambda_d}{6}$, 
substituting into \eqref{e:det}, we get
 \begin{align*}
      \left|\det\left(\nabla f^{-1} \right)\right|
      & \le \frac{1}{B^d}+\frac{\|y\|^22(1-\eta\lambda_1)\gamma_1}{B^d Br^4B^2\sum_{k=1}^d\frac{y_k^2 g_k}{2(r g_k-h_k)^3}} \\
      & \le \frac{1}{B^d}+\frac{2}{B^{d+3}\gamma_1^3
      \min_{k=1}^d\frac{g_k}{2(r g_k -h_k)^3}} \\
      & \le \frac{1}{B^d}+\frac{4r^3(1-\eta\lambda_d)^3}{B^{d+3}\gamma_1^3
      (1-\eta\lambda_1)} \\
      & \le \frac{1}{B^d}+\frac{8R^3}{B^{d+3}\gamma_1^3} \\
      & \le \frac{9R^3}{B^{d+3}\gamma_1^3}.
 \end{align*}
Since the density of the initial $v_0$ is upper bounded by $A$ and Lemma~\ref{l:fast_to_b} shows that $\|v_{T_0}\|< b < \gamma_1 + \epsilon$, for all $t\le T_0$, the density in the ring $\{v:\gamma_1-\Delta\le \|v\|\le\gamma_1+\Delta\}$ is no more than
  \[
    \bar A := \left(\frac{9\cdot 6^{d+3}R^3}{(\eta\lambda_d)^{d+3}\gamma_1^3}\right)^{T_0} A.
  \]
This implies that for all $t$,
  \begin{align*}
    \Pr[\gamma_1-\Delta\le \|v_t\|\le\gamma_1+\Delta]
      &\le 2\Delta S_{d-1}(\gamma_1+\Delta) \bar A \\
      &= 2\Delta \frac{2\pi^{d/2}}{\Gamma(d/2)}(\gamma_1+\Delta)^{d-1} \bar A\\
      &\le \Delta \frac{4\pi^{d/2}}{\Gamma(d/2)}(2\gamma_1)^{d-1} \bar A \\
      &\le \frac{\delta}{T_0}
  \end{align*}
where $S_{d-1}(r)$ is the surface area of the $(d-1)$-sphere of radius $r$ in $\Re^d$.  Since there are at most
$T_0$ iterations for which $\| v_t \| \geq b$,
there are at most $T_0$ steps for which
$\| v_t \| \geq \gamma_1 + \epsilon$,
which completes the proof.
\end{proof}

\subsection{SAM as gradient descent}
\label{s:gd}

The analysis of SAM is complicated by the presence of the $\|\Lambda w_t\|$ term, which couples all $d$ components of the recurrence. The following lemma shows that if we incorporate an alternating sign, we can view SAM as a gradient descent update based on a non-convex objective function.
Note that the lemma does not require that $\lambda_1>\lambda_2$.
\begin{lemma}
\label{l:u}
For
$u_t:=(-1)^tw_t$, if $\|w_t\|>0$ for all $t$, the iteration
    \[
      w_{t+1,i} = -\eta\rho\lambda_i^2\frac{w_{t,i}}{\|\Lambda w_t\|}
      + (1-\eta\lambda_i)w_{t,i}
    \]
for $i=1,\ldots,d$ is equivalent to 
  \[
    u_{t+1} = u_t - \eta\rho \nabla J(u_t),
  \]
where
  \[
    J(u) = \frac{1}{2} u^{\top} C  u - \left\lVert \Lambda u \right\rVert
     = \frac{1}{2} \sum_{i=1}^d \frac{\lambda_i^2 u_i^2}{\beta_i} - \sqrt{\sum_{i=1}^d \lambda_i^2 u_i^2}
  \]
with
  \[
    C = \frac{1}{\eta\rho}(2I - \eta \Lambda) = \diag\left(\frac{\lambda_1^2}{\beta_1},\ldots, \frac{\lambda_d^2}{\beta_d}\right).
  \]

Furthermore, $J$ has derivatives
  \begin{align*}
      \nabla J(u) & = C u-\frac{ \Lambda^2 u}{\|\Lambda u\|}, \\
      \nabla^2 J(u) & = C - \frac{1}{\|\Lambda u\|}\Lambda P_{\perp} \Lambda
  \end{align*}
where $P_{\perp}= I -
      \Lambda u u^\top\Lambda/\|\Lambda u\|^2$ is the projection on to the
subspace orthogonal to $\Lambda u$. 
Further, $\nabla J(u)=0$ iff 
for some $1\le i\le d$, $\|u\|=\beta_i/\lambda_i$ and $u\in\myspan\{e_j:\lambda_j=\lambda_i\}$.

Also,
  \[
    J(u_{t+1})-J(u_t)
      \le - \frac{1}{2\rho} \sum_{i=1}^d u_{t,i}^2 \left(1-\frac{\beta_i}{\|\Lambda u\|}\right)^2
      (2-\eta\lambda_i)^2\lambda_i.
  \]

For unit norm $\hat u$ satisfying $\nabla J\left(\displaystyle\frac{\beta_i}{\lambda_i} \hat u\right)=0$,
  \[
    \nabla^2 J\left(\frac{\beta_i}{\lambda_i} \hat u\right)
     = \Lambda^2 \left(\sum_{j:\beta_j\not=\beta_i} \left(\frac{1}{\beta_j}-\frac{1}{\beta_i}\right)e_je_j^\top + \frac{1}{\beta_i}e_ie_i^\top\right)
  \]
which has $|\{j:\beta_j<\beta_i\}|+1$ positive eigenvalues, $|\{j:\beta_j>\beta_i\}|$ negative eigenvalues, and
$|\{j:\beta_j=\beta_i\}|-1$ zero eigenvalues.
\end{lemma}

\begin{figure}
    \centering
    \includegraphics[width=6in]{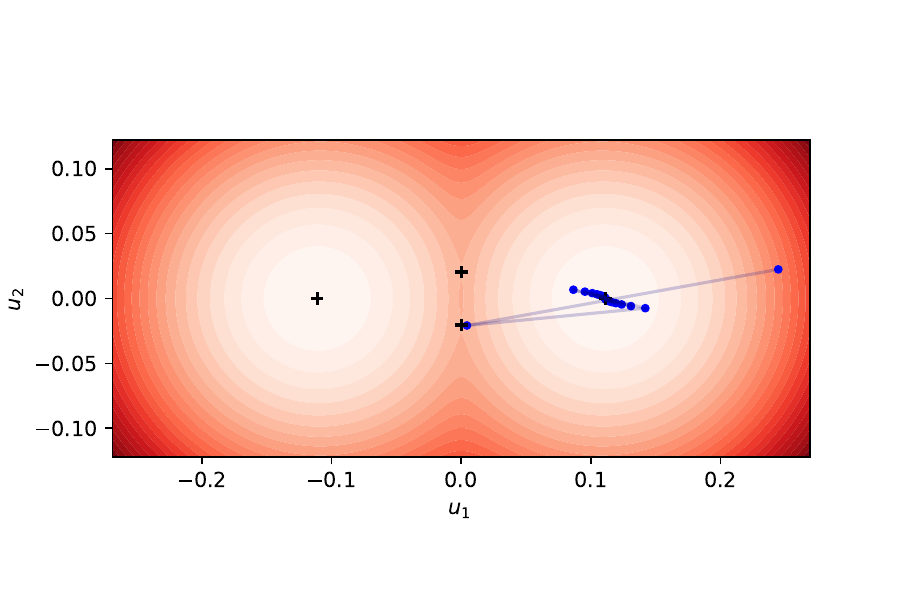}
    \caption{A heat map of the function $J$ defined
    in Lemma~\ref{l:u} in the case that
    $d=2$, $\eta = 1/5$, $\rho = 1$, $\lambda_1 = 1$
    and $\lambda_2 = 1/5$, along with
    $u_0, u_1,..., u_{25}$ when $w_0 = (0.244,0.0224)$. 
    The black pluses mark the stationary points.}
    \label{f:J}
\end{figure}
\begin{remark}
Although $J$ is not convex, it is well-behaved
(see Figure~\ref{f:J}). In particular,
the set of all stationary points with only non-negative eigenvalues is
  \[
    M=\left\{u\in\Re^d: \|u\|=\frac{\beta_1}{\lambda_1},\, u\in\myspan\{e_j:\lambda_j=\lambda_1\}\right\},
  \]
and this is the set of global minima. There are no other local minima, since at all other stationary points the Hessian has a negative eigenvalue. It is easy to see that all $u\in M$ have $J(u)=-\beta_1/2$. And, for example, if $\lambda_1>\lambda_2$, then $M=\left\{  -\frac{\beta_1}{
\lambda_1} e_1, \frac{\beta_1}{
\lambda_1} e_1\right\}$, and at all other stationary points the Hessian has at least one negative eigenvalue no larger than $1/\beta_1-1/\beta_2<0$.
\end{remark}

\begin{proof}
First notice that
  \begin{align*}
    \nabla J(u)
      & = \nabla\left(   \frac{u ^\top (I-\eta\Lambda/2) u}{\eta\rho} - \sqrt{u^\top\Lambda^2u}\right) \\
      & = \frac{2I-\eta\Lambda}{\eta\rho}u - \frac{ \Lambda^2 u}{\|\Lambda u\|},
  \end{align*}
and
  \begin{align}
    \nabla^2 J(u)
            & = \frac{ 2I-\eta\Lambda}{\eta\rho} 
              - \frac{\Lambda^2}{\|\Lambda u\|} + \frac{\Lambda^2 u u^\top\Lambda^2}{\|\Lambda u\|^3} \notag\\
      & = \frac{2I-\eta\Lambda}{\eta\rho} - \frac{1}{\|\Lambda u\|}\Lambda\left(I -
      \frac{\Lambda uu^\top\Lambda}{\|\Lambda u\|^2}\right)\Lambda \notag\\
      & \preceq \frac{ 2I-\eta\Lambda}{\eta\rho}.  \label{e:uppernabla2J}
  \end{align}
Now, from~\eqref{e:SAM-quadratic},
  \begin{align*}
      u_{t+1}
        & = (-1)^{t+1}w_{t+1} \\
        & = (-1)^{t+1}\left(I - \eta\Lambda - \frac{\eta \rho}{\|\Lambda w_t\|}\Lambda^2\right)w_t \\
        & = \eta\rho\Lambda^2\frac{u_t}{\|\Lambda u_t\|} - (I-\eta\Lambda)u_t \\
        & = u_t  - \eta\rho \left( \left( \frac{2I-\eta\Lambda}{\eta\rho} \right) u_t - \Lambda^2\frac{u_t}{\|\Lambda u_t\|}\right) \\
        & = u_t  - \eta\rho \nabla\left(\frac{ u_t^\top (2I-\eta \Lambda )  u_t }{2\eta\rho}
                - \|\Lambda u_t\|\right) \\
        & = u_t  - \eta\rho \nabla J(u_t).
  \end{align*}
Applying the fundamental theorem of calculus twice and using~\eqref{e:uppernabla2J},
  \begin{align*}
    J(u_{t+1})-J(u_t)
      & = (u_{t+1}-u_t)^\top \int_0^1\nabla J(u_t+h(u_{t+1}-u_t))\, dh \\
      & = (u_{t+1}-u_t)^\top \int_0^1\left(\nabla J(u_t) + 
      h \left( \int_0^1\nabla^2 J(u_t+x h(u_{t+1}-u_t))\,dx \right) (u_{t+1}-u_t)\right)\, dh \\
      & \le \nabla J(u_t)^\top(u_{t+1}-u_t)
        + \frac{1}{2} (u_{t+1}-u_t)^\top \frac{2I-\eta\Lambda}{\eta\rho} (u_{t+1}-u_t) \\
      & = -\eta\rho \nabla J(u_t)^\top \nabla J(u_t)
        + \eta\rho \nabla J(u_t)^\top \left(I-\frac{\eta\Lambda}{2}\right)\nabla J(u_t) \\
      & = -\eta\rho u_t^\top\left(\frac{2I-\eta\Lambda}{\eta\rho}-\frac{\Lambda^2}{\|\Lambda u\|}\right)^2u_t  \\
      & \hspace{1in}
        + \eta\rho u_t^\top \left(\frac{2I-\eta\Lambda}{\eta\rho}-\frac{\Lambda^2}{\|\Lambda u\|}\right)
        \left(I-\frac{\eta\Lambda}{2}\right)\left(\frac{2I-\eta\Lambda}{\eta\rho}-\frac{\Lambda^2}{\|\Lambda u\|}\right)u_t \\[2mm]
      & = -\eta\rho u_t^\top\left(\frac{2I-\eta\Lambda}{\eta\rho}-\frac{\Lambda^2}{\|\Lambda u\|}\right)\frac{\eta\Lambda}{2}\left(\frac{2I-\eta\Lambda}{\eta\rho}-\frac{\Lambda^2}{\|\Lambda u\|}\right)u_t \\
      & = - \eta\rho \sum_{i=1}^d u_{t,i}^2 \left(\frac{2-\eta\lambda_i}{\eta\rho}-\frac{\lambda_i^2}{\|\Lambda u\|}\right)^2
      \frac{\eta\lambda_i}{2} \\
      & = - \frac{1}{2\rho} \sum_{i=1}^d u_{t,i}^2 \left(1-\frac{\eta\rho\lambda_i^2}{(2-\eta\lambda_i)\|\Lambda u\|}\right)^2
      (2-\eta\lambda_i)^2\lambda_i \\
      & = - \frac{1}{2\rho} \sum_{i=1}^d u_{t,i}^2 \left(1-\frac{\beta_i}{\|\Lambda u\|}\right)^2
      (2-\eta\lambda_i)^2\lambda_i,
  \end{align*}
where $\beta_i = \eta\rho\lambda_i^2/(2-\eta\lambda_i)$ as before.

Now, if $u$ satisfies $\nabla J(u)=0$, then
\begin{align*}
&    \frac{(2I-\eta\Lambda) u}{\eta\rho} = \frac{\Lambda^2 u}{\|\Lambda u\|} \\
& \Rightarrow 
  \left( \frac{\Lambda^{-1}}{|| \Lambda u ||} \right) \frac{(2I-\eta\Lambda) u}{\eta\rho} = \frac{\Lambda u}{\|\Lambda u\|^2} \\
&  \Rightarrow \left( \frac{\Lambda^{-1}}{|| \Lambda u ||} \right) \frac{(2I-\eta\Lambda) \Lambda^{-1} \Lambda u}{\eta\rho} = \frac{\Lambda u}{\|\Lambda u\|^2} \\
& \Rightarrow
     \left( \frac{\Lambda^{-2} (2I-\eta\Lambda) }{\eta\rho} \right) \frac{\Lambda u}{\|\Lambda u\|} = \frac{\Lambda u}{\|\Lambda u\|^2},
\end{align*}
that is, $\Lambda u/\|\Lambda u\|$ is an eigenvector of $\frac{\Lambda^{-2} (2I-\eta\Lambda) }{\eta\rho}=\diag(1/\beta_1,\ldots,1/\beta_d)$ with eigenvalue $1/\|\Lambda u\|$.
Consider one such stationary point:
$\zeta_i e_i$, for some $i \in [d]$ and
$\zeta_i \in \R$.  We have
\[
\frac{(2 - \eta \lambda_i) \zeta_i}{\eta \rho}
 = \frac{\lambda_i^2 \zeta_i}{\lambda_i |\zeta_i|}
\]
which implies
\[
|\zeta_i| = 
  \frac{\eta \rho \lambda_i}{2 - \eta \lambda_i}
 = \frac{\beta_i}{\lambda_i}.
\]
Nearly exactly the same reasoning implies that 
$\| u \| = \frac{\beta_i}{\lambda_i}$ for all stationary points $u$
whose eigenvalues are the same as $e_i$.
For such a stationary point,
  \begin{align*}
      \nabla^2 J(u)
      &= \frac{2I-\eta\Lambda}{\eta\rho} - \frac{1}{\beta_i}\Lambda\left(I - e_ie_i^\top\right)\Lambda \\
      &= \Lambda^2 \left(\diag(1/\beta_1,\ldots,1/\beta_d)- \frac{1}{\beta_i}\left(I - e_ie_i^\top\right)\right) \\
      &= \Lambda^2 \left(\sum_{j:\beta_j\not=\beta_i} \left(\frac{1}{\beta_j}-\frac{1}{\beta_i}\right)e_je_j^\top + \frac{1}{\beta_i}e_ie_i^\top\right).
  \end{align*}
The counts of eigenvalues of different signs follow from this and the ordering $\beta_1\ge\cdots\ge\beta_d$ (Lemma~\ref{l:one_step}, part~\ref{part:five}).
\end{proof}

The following lemma shows that SAM cannot spend too much time with $\| v_t \|$ large, because $J$ is non-increasing and it decreases 
a lot
when $\| v_t \|$ is large. Lemma~\ref{l:one_step} part~\ref{part:two} shows that the norm of $v_t$ decreases when the norm is larger than $\beta_1$, and the lemma shows in particular that the norm cannot stay much larger than $\beta_1$.

\begin{lemma}
\label{l:breakaway2}
For $\epsilon > 0$, and $\| v_{T_0} \| \leq  b$,
\[
  \left|\left\{\astrut
    t\ge T_0: \| v_t \| \geq (1+\epsilon)\beta_1\right\}\right| \le
   \frac{2}{\eta \epsilon^2\lambda_1\beta_1}
   \left(
   \max_{\| \Lambda w \|\le b, s \in \{-1,1\}}J(s w) -
  \min_{u}J(u)
   \right)
  \le 
  \frac{3 \beta_1}{\eta\epsilon^2\lambda_1\beta_d}.
\]
\end{lemma}
\begin{proof}
From Lemma~\ref{l:one_step} part~\ref{part:five}, $\beta_i\le\beta_1$, and the definition of $\beta_i$ implies that $\lambda_i/\beta_i\ge\lambda_1/\beta_1$.
Thus, whenever $\| v_t \| \geq (1+\epsilon)\beta_1$,
recalling that $\eta \lambda_1 < 1$, Lemma~\ref{l:u} implies
\begin{align*}
J(u_{t+1})-J(u_t)
 & \leq - \frac{1}{2\rho} \sum_{i=1}^d u_{t,i}^2 \left(1-\frac{\beta_i}{\|\Lambda u_t \|}\right)^2
      (2-\eta\lambda_i)^2\lambda_i \\
 & = - \frac{1}{2\rho} \sum_{i=1}^d \left(\frac{v_{t,i}}{\lambda_{t,i}}\right)^2 \left(1-\frac{\beta_i}{\| v_t \|}\right)^2
      (2-\eta\lambda_i)^2\lambda_i \\
 & = - \frac{\eta}{2} \sum_{i=1}^d v_{t,i}^2 \left(1-\frac{\beta_i}{\|v_t \|}\right)^2
      \frac{(2-\eta\lambda_i) \lambda_i}{\beta_i} \\
 & \leq - \frac{\eta}{2} \sum_{i=1}^d v_{t,i}^2 \left(1-\frac{\beta_i}{\| v_t \|}\right)^2
        \frac{\lambda_i}{\beta_i} \\
 & \leq - \frac{\eta}{2} \sum_{i=1}^d v_{t,i}^2 \left(1-\frac{\beta_i}{(1 + \epsilon) \beta_1 }\right)^2
        \frac{\lambda_i}{\beta_i} \\
 & \leq - \frac{\eta}{2}\sum_{i=1}^d v_{t,i}^2 \left(1-\frac{\beta_1}{(1 + \epsilon) \beta_1 }\right)^2
        \frac{\lambda_1}{\beta_1} \\
  & = - \eta \left(1-\frac{1}{1+\epsilon}\right)^2
          \frac{\lambda_1}{2\beta_1} \|v_t\|^2 \\
  & \le - \eta \left(\frac{\epsilon}{1+\epsilon}\right)^2
          \frac{\lambda_1}{2\beta_1} (1+\epsilon)^2\beta_1^2 \\
  & = - \frac{\eta \epsilon^2\lambda_1\beta_1}{2},
\end{align*}
and since $J$ is always nonincreasing, this means there can be no more than
  \[
    \frac{2}{\eta \epsilon^2\lambda_1\beta_1}
      \left(
      \max_{w \in \R^d, s \in \{-1,1\}: \| \Lambda w \|\le b}J(s w) -
  \min_{u}J(u)
      \right)
  \]
iterations like this.

For the last inequality, we have
\begin{align*}
\max_{\|\Lambda w \|\le b, s}J(s w)
  & = \max_{0\le z\le b} \max_{\|\Lambda w\|=z} \left( \frac{1}{2}w^\top C w - z\right) \\
  & = \max_{0\le z\le b} \max_{\|v\|=z} \left(\frac{1}{2}v^\top\Lambda^{-1} \diag\left(\frac{\lambda_1^2}{\beta_1},\ldots, \frac{\lambda_d^2}{\beta_d}\right) \Lambda^{-1}v - z \right)\\
  & = \max_{0\le z\le b} \max_{\|v\|=z} \left(\frac{1}{2}v^\top \diag\left(\frac{1}{\beta_1},\ldots, \frac{1}{\beta_d}\right) v - z \right) \\
  & = \max_{0\le z\le b} \left(\frac{z^2}{2\beta_d} - z \right) \\
  & = \frac{b^2}{2\beta_d} - b \\
  & \le \frac{2 \beta_1^2}{\beta_d},
\end{align*}
since $b\le 2\beta_1$.

Since $\min_u J(u) = -\beta_1/2$, we have
\begin{align*}
\max_{\|u\|\le b}J(u) -  \min_u J(u)
 \leq \frac{2 \beta_1^2}{\beta_d} + \frac{\beta_1}{2}
 \leq \frac{3 \beta_1^2}{2 \beta_d},
\end{align*}
since $\beta_d \leq \beta_1$.
\end{proof}

\subsection{Avoiding non-minimal stationary points}
\label{s:avoiding}

Lemma~\ref{l:u} shows that the set of global minima of $J$ is a sphere of radius $\beta_1/\lambda_1$ in the subspace spanned by the $e_j$ with $\lambda_j=\lambda_1$. To simplify notation, we assume that $\lambda_1>\lambda_2$, so that this subspace is in the $e_1$ direction. Then to ensure that $J$ decreases to a global minimum, it suffices to keep $|\lambda_1 w_{t,1}| = |v_{t,1}|$ away from zero and $\|v_t\|\not=\beta_1$.
The following quantity measures the extent to which $v_t$ still has
``energy'' in components other than the first.
\begin{definition}
Define $\delta_t = 1 - \frac{|v_{t,1}|}{\|v_t \|}$.
\end{definition}

\begin{lemma}
\label{l:tau_by_log_odds}
We have
\[
\delta_t \leq \frac{1}{2}\frac{\sum_{i=2}^d v_{t,i}^2}{v_{t,1}^2}
\]
whenever this bound is most $1/2$.
\end{lemma}
\begin{proof}
We have
  \begin{align*}
   \delta_t
       & = 1-\frac{|v_{t,1}|}{\|v_t\|} \\
        & = 1 - \frac{1}{\sqrt{1 + \sum_{i=2}^d v_{t,i}^2/v_{t,1}^2}} \\
        & \le \frac{1}{2}\frac{\sum_{i=2}^dv_{t,i}^2}{v_{t,1}^2}
  \end{align*}
since, for all $0\le\alpha\le 1$, we have $1-1/\sqrt{1+\alpha}\le\alpha/2$. Indeed, this inequality is equivalent to
  \begin{align*}
    1 &\ge (1+\alpha)\left(1-\frac{\alpha}{2}\right)^2 \\
      &= (1+\alpha)\left(1-\alpha + \frac{\alpha^2}{4}\right) \\
      &= 1-\alpha^2 + \frac{\alpha^2}{4}+\frac{\alpha^3}{4} \\
      &= 1-\frac{\alpha^2(3-\alpha)}{4}.
  \end{align*}
\end{proof}

Lemma~\ref{l:one_step} part~\ref{part:four} shows that the first component increases relative to the other components when $\|v_t\|<\alpha_d$. But as long as $\lambda_d>0$, part~\ref{part:three} shows that $\alpha_d>\beta_1$, and in that case Lemma~\ref{l:breakaway2} implies that $\|v_t\|$ does not spend too much time above $\alpha_d$. Our assumption that $\lambda_1>\lambda_2$ ensures that the first component of $v_t$ increases in magnitude relative to all the other components; otherwise, the equations describing the evolution of the first and second components are identical. The key constant depends on both $\lambda_d$ and the gap between $\lambda_1$ and $\lambda_2$.

\begin{lemma}
\label{l:p_one_bigger}
Define
  \[
    \mu = \min\left\{\eta\lambda_d,
    \frac{\lambda_1^2}{\lambda_2^2}-1\right\}.
  \]
If $v_{t,1}^2>0$, the following two statements are equivalent:
  \begin{align*}
      \frac{v_{t+1,i}^2}{v_{t+1,1}^2}
        &< \frac{1}{(1+\mu)^2} \frac{v_{t,i}^2}{v_{t,1}^2}\qquad\forall i\in\{2,\ldots,d\}, \\
      \frac{\|v_t\|}{\beta_1}
        & < \frac{2-\eta\lambda_1}{2-\eta\lambda_1
        -(\eta\lambda_d-\mu)-\eta\lambda_d\mu}
        \left(1+(1+\mu)\frac{\lambda_d^2}{\lambda_1^2}\right).
  \end{align*}
Thus, if $v_{t,1}^2>0$ for all $t$,
\[
  \left|\left\{t: \|v_t\|\le b\text{ and for some }i\in\{2,\ldots,d\},\,
  \frac{v_{t+1,i}^2}{v_{t+1,1}^2}
    \ge \frac{1}{(1+\mu)^2} \frac{v_{t,i}^2}{v_{t,1}^2} \right\}\right| \le T_1,
\]
where
\[
  T_1 = \frac{3 \beta_1 \lambda_1^3}
     {\eta \beta_d \lambda_d^4} .
\]
Furthermore, if $T_0$ is such that $\|v_{T_0}\|\le b$, then
  \[
    \delta_{T+1} \le \frac{1}{2}\left(\frac{1}{1+\mu}\right)^{2(T-T_1)}
    \left(\frac{2}{\eta\lambda_1}\right)^{2T_1}
    \frac{\sum_{i=2}^d v_{T_0,i}^2}{v_{T_0,1}^2},
  \]
provided that $T$ is large enough that this upper bound is less than $1/2$.

Thus, for all $\epsilon < 1/2$,
$
    \delta_{T+1} \le \epsilon
$
provided
  \[
    T\ge \frac{2}{\mu}
        \left(
        T_1 \log \left(
                    \frac{4}
                         { \eta \lambda_1}
                   \right)
         +\frac{1}{2} \log\left( \frac{\sum_{i = 2}^d v_{T_0}^2}
                       {2 \epsilon v_{T_0,1}^2} \right)
       \right).
  \]
\end{lemma}

\begin{proof}
For the equivalence, notice that the evolution of $v_t$ implies that
  \[
      \frac{v_{t+1,i}^2}{v_{t+1,1}^2}
        < \frac{1}{(1+\mu)^2} \frac{v_{t,i}^2}{v_{t,1}^2}
  \]
if and only if
  \begin{equation}\label{equation:ucomp}
    (1-\eta\lambda_1)^2(\|v_t\|-\gamma_1)^2 > (1+\mu)^2(1-\eta\lambda_i)^2
    (\|v_t\|-\gamma_i)^2.
  \end{equation}
We can apply Lemma~\ref{lemma:compare-quadratics}, because $0<1-\eta\lambda_1<(1+\mu)(1-\eta\lambda_i)$,
$\gamma_1>\gamma_i$, and
  \begin{align*}
    &&&(1-\eta\lambda_1)\gamma_1\ge(1+\mu)(1-\eta\lambda_i)\gamma_i \\
    &\Leftrightarrow&&
    \lambda_1^2\ge(1+\mu)\lambda_i^2 \\
    &\Leftrightarrow&&
    \frac{\lambda_1^2}{\lambda_i^2}-1\ge\mu,
  \end{align*}
which follows from the definition of $\mu$.
Lemma~\ref{lemma:compare-quadratics} implies that when $\|v_t\|>0$,
\eqref{equation:ucomp} is equivalent to
  \begin{align*}
    \|v_t\|
    & <\frac{(1-\eta\lambda_1)\gamma_1 + (1+\mu)(1-\eta\lambda_i)\gamma_i}{1-\eta\lambda_1+(1+\mu)(1-\eta\lambda_i)}.
  \end{align*}
Because the right hand side is a convex combination of $\gamma_1$ and $\gamma_i$, because $\gamma_d\le\cdots\le\gamma_2$, and because the convex coefficients are also ordered
($1-\eta\lambda_2\le \cdots\le 1-\eta\lambda_d$), these inequalities for all $i$ are implied by the corresponding inequality for $i=d$, which is
  \begin{align*}
    \|v_t\|
    & <\frac{(1-\eta\lambda_1)\gamma_1 + (1+\mu)(1-\eta\lambda_d)\gamma_d}{1-\eta\lambda_1+(1+\mu)(1-\eta\lambda_d)} \\
    \Leftrightarrow\qquad
    \frac{\|v_t\|}{\beta_1}
    & <\frac{2-\eta\lambda_1}{(1-\eta\lambda_1)\gamma_1}
     \left(
      \frac{(1-\eta\lambda_1)\gamma_1 + (1+\mu)(1-\eta\lambda_d)\gamma_d}{1-\eta\lambda_1+(1+\mu)(1-\eta\lambda_d)}
      \right) 
      \\
    & = \frac{2-\eta\lambda_1}{2-\eta\lambda_1-(\eta\lambda_d-\mu)-\eta\lambda_d\mu}\left(1+\frac{(1+\mu)(1-\eta\lambda_d)\gamma_d}{(1-\eta\lambda_1)\gamma_1}\right) \\
    & = \frac{2-\eta\lambda_1}{2-\eta\lambda_1-(\eta\lambda_d-\mu)-\eta\lambda_d\mu}\left(1+(1+\mu)\frac{\lambda_d^2}{\lambda_1^2}\right).
  \end{align*}
This proves the first part of the lemma.

Hence, for each iteration when, for some $2\le i\le d$,
\[
\frac{v_{t+1,i}^2}{v_{t+1,1}^2}
    \ge \frac{1}{(1+\mu)^2} \frac{v_{t,i}^2}{v_{t,1}^2},
\]
we must have
  \[
      \frac{\|v_t\|}{\beta_1}
        \ge \frac{2-\eta\lambda_1}{2-\eta\lambda_1-(\eta\lambda_d-\mu)-\eta\lambda_d\mu}
        \left(1+(1+\mu)\frac{\lambda_d^2}{\lambda_1^2}\right)> 1+(1+\mu)\frac{\lambda_d^2}{\lambda_1^2},
  \]
where the inequality follows from $0<\mu\le \eta\lambda_d$
and $\eta < 1/\lambda_1$.
Lemma~\ref{l:breakaway2} implies that the number of these iterations for which we also have $\|v_t\|\le b$ is no more than
\begin{align*}
& \frac{3 \beta_1}
     {\eta 
     \left(
     (1 + \mu) \frac{\lambda_d^2}{\lambda_1^2}
     \right)^2
      \lambda_1 \beta_d} 
%
 \leq \frac{3 \beta_1 \lambda_1^3}
     {\eta \beta_d \lambda_d^4} 
 = T_1.    
\end{align*}

For the third part, consider the sequence of steps from $t=T_0$ to $t=T\ge T_1$. There are at least $T-T_1$ steps when
  \[
    \frac{v_{t+1,i}^2}{v_{t+1,1}^2} < \frac{1}{(1+\mu)^2}
    \frac{v_{t,i}^2}{v_{t,1}^2},
  \]
and no more than $T_1$ steps when this fails, and for those steps we have
  \begin{align*}
    \frac{v_{t+1,i}^2}{v_{t+1,1}^2}
    & = \frac{(1-\eta\lambda_i)^2(\|v_t\|-\gamma_i)^2}
    {(1-\eta\lambda_1)^2(\|v_t\|-\gamma_1)^2}
    \frac{v_{t,i}^2}{v_{t,1}^2} \\
    & \le \frac{(1-\eta\lambda_d)^2\gamma_1^2}
    {(1-\eta\lambda_1)^2( b-\gamma_1)^2}
    \frac{v_{t,i}^2}{v_{t,1}^2} \\
    & = \frac{(1-\eta\lambda_d)^2}{(1-\eta\lambda_1)^2\eta^2\lambda_1^2}
    \frac{v_{t,i}^2}{v_{t,1}^2}.
  \end{align*}
(We used the fact that $0\le\|v_t\|\le b\le\gamma_1$, and so $(\gamma_i-\|v_t\|)^2\le\gamma_1^2$.)
So we have
\begin{align*}
    \frac{\sum_{i=2}^d v_{T+1,i}^2}{v_{T+1,1}^2}
    &\le \left(\frac{1}{(1+\mu)^2}\right)^{T-T_1}
    \left(\frac{(1-\eta\lambda_d)^2}{(1-\eta\lambda_1)^2\eta^2\lambda_1^2}\right)^{T_1}
    \frac{\sum_{i=2}^d v_{T_0,i}^2}{v_{T_0,1}^2}.
  \end{align*}
Applying Lemma~\ref{l:tau_by_log_odds}
\begin{align*}
    \delta_t
    &\le \frac{1}{2} \left(\frac{1}{(1+\mu)^2}\right)^{T-T_1}
    \left(\frac{(1-\eta\lambda_d)^2}{(1-\eta\lambda_1)^2\eta^2\lambda_1^2}\right)^{T_1}
    \frac{\sum_{i=2}^d v_{T_0,i}^2}{v_{T_0,1}^2},
  \end{align*}
if this bound is at most $1/2$.  Solving for $T$, for all $0 < \epsilon < 1/2$,
for all 
\[
    T\ge \frac{1}{\log(1 + \mu)}
        \left(
        T_1 \log \left(
                    \frac{(1 + \mu) (1 - \eta \lambda_d)}
                         {(1 - \eta \lambda_1) \eta \lambda_1}
                   \right)
         +\frac{1}{2} \log\left( \frac{\sum_{i = 2}^d v_{T_0}^2}
                       {2 \epsilon v_{T_0,1}^2} \right)
       \right),
\]
we have $\delta_t \leq \epsilon$.
Noting that $\mu = \min\left\{\eta\lambda_d,
    \frac{\lambda_1^2}{\lambda_2^2}-1\right\} \leq \eta\lambda_d \leq \eta \lambda_1 \leq 1/2$ completes the proof.
\end{proof}

Once the first component of $v_t$ dominates, the recurrence becomes essentially one-dimensional, and its convergence is easier to analyze, as the following lemma shows.

\begin{definition}
Let $s_t = \sign(v_{t,1})$.
\end{definition}

\begin{lemma}
\label{l:v_one_closer}
If $\|v_t\|>0$,
  \begin{align*}
    v_{t+1,1}-(-s_t \beta_1)
    &= -(1-\eta\lambda_1) \left( v_{t,1}-s_t\beta_1 
      + s_t\gamma_1 \delta_t \right).
  \end{align*}
If $0 < || v_T || \leq b$ and 
for all $t\ge T$, $\delta_t\le \frac{\eta \lambda_1 \beta_1}{2}$, 
then for all $t\ge T$,
  \begin{align*}
  \left|v_{t+1,1}-(-1)^{t+1 -T} s_T \beta_1\right|
    &\le (1-\eta\lambda_1)\left(\left|v_{t,1}- (-1)^{t-T} s_T\beta_1\right| + \gamma_1 \delta_t\right).
  \end{align*}
\end{lemma}

\begin{proof}
  From the recurrence for $v_t$, we have
  \begin{align*}
  v_{t+1,1}
    & = (1-\eta\lambda_1)\left(1- \frac{\gamma_1}{\lv v_t \rv}\right) v_{t,1} \\
      &= (1-\eta\lambda_1) v_{t,1} - (2-\eta\lambda_1)s_t\beta_1\frac{ |v_{t,1}|}{\|v_t\|}   \\
       &= (1-\eta\lambda_1)v_{t,1} - s_t\beta_1\left(1-(2-\eta\lambda_1)
        \left(1-\frac{\left|v_{t,1}\right|}{\|v_t\|}\right)\right) + s_t\beta_1(1-(2-\eta\lambda_1)) \\
    & = (1-\eta\lambda_1)(v_{t,1}-s_t\beta_1)
      - s_t\beta_1\left(1-(2-\eta\lambda_1)\delta_t\right).
  \end{align*}
  So
  \begin{align}
   \nonumber
      v_{t+1,1}- (-s_t\beta_1)
      &= -(1-\eta\lambda_1)\left(v_{t,1}-s_t\beta_1\right)  
      + s_t\beta_1(2-\eta\lambda_1)\delta_t \\
   \nonumber
      &= -(1-\eta\lambda_1)\left(v_{t,1}-s_t\beta_1\right)  
      + s_t\gamma_1(1-\eta\lambda_1)\delta_t \\
   \label{e:equality}
      &= -(1-\eta\lambda_1) \left( v_{t,1}-s_t\beta_1 
      + s_t\gamma_1 \delta_t \right),
  \end{align}
  which is the equality in the lemma.

Since $|| v_T || \leq b$, 
Part 5 of Lemma~\ref{l:one_step} implies that
for all
$t \geq T$, $|| v_t || \leq b$, which
in turn implies $0 \leq s_t v_{t,1} \leq b$.
Since $0 \leq b/2 < \beta_1$, 
and $|v_{t,1} - s_t \beta_1| = |s_t v_{t,1} - \beta_1|$,
this implies, for all $t \geq T$,
\begin{align}
\label{e:v_t_one_close}
| v_{t,1} - s_t \beta_t | 
  \leq b - \beta_1 
  = (1 - \eta \lambda_1) \beta_1.
\end{align}

By the triangle inequality for the absolute difference, since $\delta_t\ge 0$,
\begin{align*}
  \left|v_{t+1,1}-(-s_t \beta_1)\right|
    &\le (1-\eta\lambda_1)\left(\left|v_{t,1}-s_t\beta_1\right| + \gamma_1 \delta_t\right), 
\end{align*}
which in turn implies
\begin{align*}
  \min\left\{\left|v_{t+1,1}-\beta_1\right|,\left|v_{t+1,1}+\beta_1\right|\right\}
   &\le (1-\eta\lambda_1)\left(\left|v_{t,1}-s_t\beta_1\right| + \gamma_1 \delta_t\right).
\end{align*}
Because $\beta_1>0$,
   \[
     \left|v_{t+1,1}-s_{t+1} \beta_1\right| = \min\left\{\left|v_{t+1,1}-\beta_1\right|,\left|v_{t+1,1}+\beta_1\right|\right\},
   \]
so
\begin{align}
\label{e:abs_progress}
  \left|v_{t+1,1}-s_{t+1} \beta_1\right|
   &\le (1-\eta\lambda_1)\left(\left|v_{t,1}-s_t\beta_1\right| + \gamma_1 \delta_t\right).
\end{align}

It remains to show that, for all $t \geq T$,
if $\delta_t \leq \frac{\eta \lambda_1 \beta_1}{2}$, then $s_{t+1} = -s_t$.  

To see this, assume as a first case that $s_t = 1$.  Then
\eqref{e:equality} implies
\begin{align*}
v_{t+1,1}
  & = -\beta_1 - (1-\eta\lambda_1) \left( v_{t,1}- \beta_1 
      + \gamma_1 \delta_t \right) \\
  & \leq -\beta_1 + (1-\eta\lambda_1) \left( |v_{t,1}- \beta_1| 
      + \gamma_1 \delta_t \right) \\
  & \leq -\beta_1 + (1-\eta\lambda_1) 
   \left( (1-\eta\lambda_1) \beta_1
      + \gamma_1 \delta_t \right)
    \hspace{1in} \mbox{(by \eqref{e:v_t_one_close})}
         \\
  & < 0,
\end{align*}
since $\delta_t\le \frac{\eta \lambda_1 \beta_1}{2}$,
so $s_{t+1} = -1$.

Similarly, if $s_t = -1$, then
\begin{align*}
v_{t+1,1}
  & = \beta_1 - (1-\eta\lambda_1) \left( v_{t,1}- \beta_1 
      + \gamma_1 \delta_t \right) \\
  & \geq \beta_1 - (1-\eta\lambda_1) \left( |v_{t,1}- (-\beta_1)| 
      + \gamma_1 \delta_t \right) \\
  & > 0, \\
\end{align*}
so $s_{t+1} = 1$.  
The last
inequality of the lemma then follows by induction.
\end{proof}

\begin{lemma}
\label{l:v_one_close_and_p_one_close}
If 
$T_0$ is the first iteration where $\| v_{T_0} \| \leq  b$, then, for 
all 
\[
0 < \epsilon < \min \left\{ \sqrt{\frac{\eta \lambda_1 \beta_1}{2 \gamma_1}}, \frac{\eta \lambda_1 }{2 \gamma_1}, \frac{1}{b}, \frac{\beta_1}{2} \right\},
\]
for
\begin{align*}
& T_2 = 
  \frac{2}{\mu}
        \left(
     \frac{3 \beta_1 \lambda_1^3}
     {\eta \beta_d \lambda_d^4} 
      \log \left(
                    \frac{4}
                         {\eta \lambda_1}
                   \right)
         +\frac{1}{2} \log\left( \frac{\sum_{i = 2}^d v_{T_0}^2}
                       {2 \epsilon^2 v_{T_0,1}^2} \right)
       \right)
     + \frac{6}{\eta \lambda_1} \ln \left( \frac{1}{\epsilon} \right),
\end{align*}
for all $t \geq T_2$, we have
\[
| v_{t,1} - (-1)^{t - T_2} s_{T_2} \beta_1 | \leq \epsilon
  \mbox{ and } \delta_t \leq \epsilon^2.
\]
\end{lemma}
\begin{proof}
The last inequality
of Lemma~\ref{l:p_one_bigger} implies that,
for
\[
t^* \eqdef \left\lceil \frac{2}{\mu}
        \left(
       \frac{3 \beta_1 \lambda_1^3}
     {\eta  \beta_d \lambda_d^4} 
      \log \left(
                    \frac{4}
                         {\eta \lambda_1}
                   \right)
         +\frac{1}{2} \log\left( \frac{\sum_{i = 2}^d v_{T_0}^2}
                       {2 \epsilon^2 v_{T_0,1}^2} \right)
       \right)
   \right\rceil,
\] 
we have
\begin{equation}
\forall t \geq t^*,
\delta_t \leq \epsilon^2.
\end{equation}

For all $t \geq t^*$, since $\delta_t \leq \epsilon^2 \leq \frac{\eta \lambda_1 \beta_1}{2 \gamma_1}$, by Lemma~\ref{l:v_one_closer},
we have
\begin{align*}
 \left| v_{t+1,1} - (-1)^{t+1 - t^*} s_{t^*} \beta_1\right|
    &\le (1-\eta\lambda_1) \left| v_{t,1}  - (-1)^{t - t^*} s_{t^*} \beta_1\right| + \gamma_1 \epsilon^2. \\
\end{align*}
If $\left| v_{t,1}  - (-1)^{t - t^*} s_{t^*} \beta_1\right| > \epsilon$, this implies
\begin{align*}
 \left| v_{t+1,1} - (-1)^{t+1 - t^*} s_{t^*} \beta_1\right|
  & \leq \left(1-\eta\lambda_1 + \gamma_1 \epsilon \right) \left| v_{t,1}  - (-1)^{t - t^*} s_{t^*} \beta_1\right|. \\
\end{align*}
Since $\epsilon \leq \frac{\eta \lambda_1 }{2 \gamma_1}$, this yields
\begin{align*}
\left| v_{t+1,1} - (-1)^{t+1 - t^*} s_{t^*} \beta_1\right|
    & \leq \left(
    1 - \frac{\eta \lambda_1}{2} \right)
    \left| v_{t,1}  - (-1)^{t - t^*} s_{t^*} \beta_1\right|.
\end{align*}
Since $\| v_{t^*} \| \leq b$, $| v_{t^*,1} | \leq b$, which,
since $\beta_1 \leq b$, implies $\left| v_{t^*,1}  -  s_{t^*} \beta_1\right| \leq  b$.
Thus, by induction, for
all $t \geq t^*$, we have
\[
\left| v_{t+1,1} - (-1)^{t+1 - t^*} s_{t^*} \beta_1\right| \leq \left(
    1 - \frac{\eta \lambda_1}{2} \right)^{t - t^*} b.
\]
Thus, if $t \geq T_2 = t^* + \frac{2}{\eta \lambda_1} \ln \left( \frac{  b}{\epsilon} \right)$,
we get $\left| v_{t+1,1} - (-1)^{t+1 - t^*} s_{t^*} \beta_1\right|  \leq \epsilon$.
Since $\epsilon < \beta/2$, this implies $s_{t+1} = \sign(v_{t+1,1}) = (-1)^{t+1 - t^*} s_{t^*}$.
Since, $\epsilon \leq 1/b$, this completes the proof.
\end{proof}

\begin{lemma}
\label{l:v_one_and_pone_close_implies_vector_close}
For all $0 < \epsilon \leq 1$, if
$| v_{t,1} - (-1)^{t - T_2} s_{T_2} \beta_1 | \leq \epsilon$
and $\delta_t \leq \epsilon^2$,
then 
\[
\| v_t - (-1)^{t - T_2} s_{T_2} \beta_1 e_1 \| \leq 
2 (1 + \beta_1) \epsilon.
\]
\end{lemma}
\begin{proof}
If $\delta_t \leq \epsilon^2$, then
\begin{equation}
\label{e:pone}
\frac{v_{t,1}^2}{\lv v_t \rv^2} \geq (1 - \epsilon^2)^2.
\end{equation}

We have
\begin{align*}
& \| v_t - (-1)^{t - T_2} s_{T_2} \beta_1 e_1 \|^2 \\
& = (v_{t,1} - (-1)^{t - T_2} s_{T_2} \beta_1)^2 + \sum_{i > 2} v_{t,i}^2 \\
& \leq  \epsilon^2 + \sum_{i > 2} v_{t,i}^2 \\
& =  \epsilon^2 + \lv v_t \rv^2 - v_{t,1}^2 \\
& \leq  \epsilon^2 + \left( \frac{1}{(1 - \epsilon^2)^2} - 1 \right)  v_{t,1}^2
  \;\;\;\;\;\mbox{(by \eqref{e:pone})}
   \\
& \leq  \epsilon^2 (1  + 3 v_{t,1}^2) 
    \;\;\;\;\;\mbox{(since $0 < \epsilon \leq 1$)}
   \\
& \leq  \epsilon^2 (1  + 3 (\epsilon + \beta_1)^2) \\
\end{align*}
since $| v_{t,1} - (-1)^{t - T_2} s_{T_2} \beta_1 | \leq \epsilon$.  
Since $\sqrt{1  + 3 (1 + x)^2)} \leq 2(1+x)$
for all $x > 0$, this completes the proof.
\end{proof}

\begin{lemma}
\label{l:vector_v_close}
If $T_0$ is the first iteration where $\| v_{T_0} \| \leq  b$, and $T_2$ is defined as in Lemma~\ref{l:v_one_close_and_p_one_close}, then, for 
all
$0 < \epsilon < \min \left\{ \sqrt{\frac{2 \eta \lambda_1 \beta_1}{\gamma_1}}, \frac{\eta \lambda_1 }{2 \gamma_1}, \frac{2}{b}, \beta_1, 1 \right\}$
for any 
\begin{align*}
& t \geq   \frac{2}{\mu}
        \left(
       \frac{3 \beta_1 \lambda_1^3}
     {\eta  \beta_d \lambda_d^4} 
      \log \left(
                    \frac{4}
                         {\eta \lambda_1}
                   \right)
         +\frac{1}{2} \log\left( \frac{2 (1 + \beta_1)^2  \sum_{i = 2}^d v_{T_0}^2}
                       {\epsilon^2 v_{T_0,1}^2} \right)
       \right)
     + \frac{6}{\eta \lambda_1} \ln \left( \frac{ 2 (1 + \beta_1) }{\epsilon} \right)
\end{align*}
we have
\[
\| v_t - (-1)^{t - T_2} s_{T_2} \beta_1 e_1 \| \leq \epsilon.
\]
\end{lemma}
\begin{proof}
Combine Lemmas~\ref{l:v_one_close_and_p_one_close} and \ref{l:v_one_and_pone_close_implies_vector_close}.
\end{proof}

\begin{lemma}
\label{l:v_close_implies_w_close}
For any $s \in \{-1, 1\}$,
and any $t$,
\[
\left\lv w_t -  \frac{s \beta_1 e_1}{\lambda_1} \right\rv \leq \frac{\| v_t - s \beta_1 e_1 \|}{\lambda_d}.
\]
\end{lemma}
\begin{proof}
Since $w_t = \Lambda^{-1} v_t$, we have
\begin{align*}
\left\lv w_t - \frac{s\beta_1 e_1}{\lambda_1} \right\rv 
& = \lv \Lambda^{-1} v_t -  \Lambda^{-1} s \beta_1 e_1 \rv \\
& \leq \lv \Lambda^{-1} \rv \lv v_t - s\beta_1 e_1 \rv \\
& = \frac{1}{\lambda_d} \lv v_t - s \beta_1 e_1 \rv 
\end{align*}
\end{proof}

\subsection{Putting it together}

In this subsection, we combine the lemmas proved
in earlier subsections to prove Theorem~\ref{t:quadratic}.
For $\Lambda = \diag(\lambda_1,...,\lambda_d)$,
our analysis tracks the evolution of
$v_t = \nabla \ell(w_t) = \Lambda w_t$.

By assumption, with probability $1 - \delta$,
$\lv w_0 \rv \leq R$ and
$w_{0,1}^2 \geq q$.  Let us assume from here on that this
is the case.  This implies $\lv v_0 \rv \leq \lambda_1 R$ and $v_{0,1}^2 \geq \lambda_1^2 q$.

Let $T_0$ be the index of the first iteration 
that $\| v_t \| \leq b$ holds.

Lemmas~\ref{l:fast_to_b} and \ref{l:far_from_gammaone}
imply that, with probability $1 - 2 \delta$, for
$\Delta$ defined as in Lemma~\ref{l:far_from_gammaone}, we have
\begin{align}
\nonumber
    \log\left(\frac{\sum_{i=2}^d v_{T_0,i}^2}{v_{T_0,1}^2}\right) 
     & \le
    \frac{2}{\eta\lambda_d}\left[ \log\left(\frac{\lambda_1 R}{ b}\right)\right]_+\log\left(\frac{2 \lambda_1 R}{\Delta}\right) +
    \log\left(\frac{\lambda_1^2 R^2}{v_{0,1}^2}\right) \\
    & \le
\label{e:first_component_stays_big}
    \frac{2}{\eta\lambda_d}\left[ \log\left(\frac{\lambda_1 R}{ b}\right)\right]_+\log\left(\frac{2 \lambda_1 R}{\Delta}\right) +
    \log\left(\frac{R^2}{q}\right).
\end{align}
Let us assume for the rest of this proof that this is the case.

Combining \eqref{e:first_component_stays_big} with Lemma~\ref{l:vector_v_close}, for all
\begin{align*}
& t \geq 
     \frac{6 \beta_1 \lambda_1^3}
     {\eta 
     \mu \beta_d \lambda_d^4} 
      \log \left(
                    \frac{4}
                         {\eta \lambda_1}
                   \right)
   \\
& \hspace{0.2in}
         +\frac{1}{\mu} 
           \left(
         \log\left(  \frac{2 (1 + \beta_1)^2}
                       {\lambda_d^2 \epsilon^2}\right)
          +
       \frac{2}{\eta\lambda_d} 
        \left[ \log\left(\frac{\lambda_1 R}{ b}\right) \right]_+
       \log\left(\frac{ 2 \lambda_1 R}{ \Delta}\right) +
    \log\left(\frac{R^2}{q} \right)
        \right)
        \\
& \hspace{0.2in}
     + \frac{6}{\eta \lambda_1} \ln \left( \frac{ 2 (1 + \beta_1) }{\lambda_d \epsilon} \right)
\end{align*}
we have
\begin{align}
\label{e:vector_v_close}
\| v_t - (-1)^{t - T_2} s_{T_2} \beta_1 e_1 \| \leq \lambda_d \epsilon.
\end{align}
Applying Lemma~\ref{l:far_from_gammaone} to bound
$\log\frac{1}{\Delta}$, we get that
\begin{align*}
& t \geq     \frac{6 \beta_1 \lambda_1^3}
     {\eta 
     \mu \beta_d \lambda_d^4} \log \left(
                    \frac{4}
                         {\eta \lambda_1}
                   \right) \\
& \hspace{0.2in}
         +\frac{1}{\mu} 
           \left(
         \log\left(  \frac{2 (1 + \beta_1)^2}
                       {\lambda_d^2 \epsilon^2}\right)
          +
    \log\left(\frac{R^2}{q}\right)
        \right)
        \\
& \hspace{0.2in}
         +
       \frac{2}{\eta\lambda_d \mu}\left[\log\left(\frac{\lambda_1 R}{ b}\right)\right]_+
                              \Bigg(
                              \log\left(2 \lambda_1 R
                              \right) 
                              + 
                              \frac{\left[\log\left(\lambda_1 R/ b\right)\right]_+ \log\left(\frac{9\cdot 6^{d+3}\lambda_1^3 R^3}{(\eta\lambda_d)^{d+3}\gamma_1^3}\right)}{\eta\lambda_d}  \\
& \hspace{2.0in} + \log\left(\frac{4\pi^{d/2}(2\gamma_1)^{d-1}[\log(\lambda_1 R/ b)]_+ A}{\Gamma(d/2)\delta \eta\lambda_d}\right)
                              \Bigg)
        \\
& \hspace{0.2in}
     + \frac{6}{\eta \lambda_1} \ln \left( \frac{ 2 (1 + \beta_1) }{\lambda_d \epsilon} \right)
\end{align*}
suffices for \eqref{e:vector_v_close}.  
Substituting the
values of $\mu$, $\beta_1$, $\beta_d$, $\gamma_1$ and $b$, simplifying and overapproximating,
we get that
\begin{align*}
& t \geq \frac{6 \lambda_1^5}
     {\eta \lambda_d^6 \min\left\{\eta\lambda_d,
    \frac{\lambda_1^2}{\lambda_2^2}-1\right\}} \log \left(
                    \frac{4 }
                         { \eta \lambda_1}
                   \right) \\
& \hspace{0.2in}
         +\frac{1}{ \min\left\{\eta\lambda_d,
    \frac{\lambda_1^2}{\lambda_2^2}-1\right\}} 
           \left(
         \log\left(  \frac{4 (1 + \eta \rho \lambda_1^2)^2}
                       {\lambda_d^2 \epsilon^2}\right)
          +
    \log\left(\frac{R^2}{q}\right)
        \right)
        \\
& \hspace{0.2in}
         +
       \frac{2 \left[\log\left(\frac{R}{ \eta \rho \lambda_1}\right)\right]_+}{\eta\lambda_d \min\left\{\eta\lambda_d,
    \frac{\lambda_1^2}{\lambda_2^2}-1\right\}}
                              \Bigg(
                              \log\left( 2 \lambda_1 R \right)
                              + 
                                \frac{\left[\log\left(\frac{ R}{\eta \rho \lambda_1}\right)\right]_+ \log\left(
                                  \frac{9\cdot 6^{d+3} R^3}
                                     {(\eta\lambda_d)^{d+3}(\eta \rho \lambda_1)^3}\right)}{\eta\lambda_d}  \\
& \hspace{2.2in} + \log\left(\frac{4\pi^{d/2}(4 \eta \rho \lambda_1^2)^{d-1}\left[ \log\left(\frac{ R}{\eta \rho \lambda_1}\right)\right]_+ A}{\Gamma(d/2)\delta \eta\lambda_d}\right)      
                              \Bigg)
        \\
& \hspace{0.2in}
     + \frac{6}{\eta \lambda_1} \ln \left( \frac{ 2 (1 + \eta \rho \lambda_1^2) }{\lambda_d \epsilon} \right)
\end{align*}
suffices.

Applying Lemma~\ref{l:v_close_implies_w_close} completes the proof.

\section{Drifting Towards Wide Minima}\label{s:nonquad}

We have seen that when SAM is applied to a convex quadratic objective, it converges to an oscillation that bounces across the minimum in the direction of greatest curvature.
In this section, we consider the behavior of SAM when it is applied to a smooth objective $\ell$ whose Hessian may vary.
Consider a point $w_z\in\Re^d$ in a $d$-dimensional parameter space that is a local minimum of $\ell$, $\nabla \ell(w_z)=0$.
For notational convenience, we assume that
  \[
    H := \nabla^2\ell(w_z) = \diag(\lambda_1,\ldots,\lambda_d).
  \]
In the neighborhood of $w_z$, the smooth objective $\ell$ can be approximated locally by the quadratic objective
  \[
    \ell_q(w) = \ell(w_z) + \frac{1}{2} (w-w_z)^\top H (w-w_z).
  \]
We are particularly interested in the overparameterized setting typical of deep learning, that is, where there are many directions in parameter space that do not affect the training objective.
Suppose, in particular, that $\lambda_1>\lambda_2\ge\cdots\ge\lambda_k >\lambda_{k+1} = \cdots = \lambda_d = 0$ for $k>1$.
Then since this quadratic objective does not vary in the $e_{k+1},\ldots,e_d$ directions, for a point $w_0$ satisfying $e_i^\top(w_0-w_z)=0$ for $i=k+1,\ldots,d$, if we initialize SAM at $w_0$ and apply it to the quadratic objective $\ell_q$, then it is clear that the condition $e_i^\top(w_t-w_z)=0$ for $i>k$ continues to hold for all $t$.
Thus, the result above shows that SAM converges to the set
  \[
    \left\{w_z\pm\frac{\beta_1}{\lambda_1}e_1\right\}.
  \]
The following theorem considers SAM's behavior on the smooth objective $\ell$ at these points.
It shows that SAM's gradient steps have a component that maintains the oscillation in the $e_1$ direction, a second-order component in the downhill direction of the spectral norm of the Hessian, and a third-order component that is small if the third derivative changes slowly. For a symmetric matrix $M$, $\lambdamax(M)$ denotes the maximum eigenvalue of $M$.
In this section, we write $D^i$ as the symmetric, multilinear, $i$th-derivative operator and $\nabla^1$ and $\nabla^2$ as the vector and matrix representations of the operators $D^1$ and $D^2$ in the canonical basis $e_1,\ldots,e_d$.

\begin{theorem}
\label{t:drift}
Suppose that $\ell$ is in $C^3$, that $D^3 \ell$ is $B$-Lipschitz with respect to the Euclidean norm and the operator norm,
and that $w_z\in\Re^d$ satisfies $\nabla \ell(w_z)=0$ and $\nabla^2\ell(w_z)=\sum_{i=1}^d\lambda_i e_i\otimes e_i$.
For $s_t \in \{-1,1\}$,
consider the point
  \[
    w_t=w_z + \frac{s_t \beta_1}{\lambda_1} e_1= w_z+\frac{\eta \rho\lambda_1 s_t}{2-\eta\lambda_1} e_1.
  \]
Then, if
$B \eta \rho \leq 1$,
SAM's update on $\ell$ gives
  \begin{align*}
    w_{t+1}-w_t
 & =    -2\frac{\eta \rho\lambda_1 s_t}{2-\eta\lambda_1} e_1 -
          \frac{\eta\rho^2}{2}\left( 1+
            \frac{\eta\lambda_1}{2-\eta\lambda_1}
            \right)^2
          \nabla\lambdamax(\nabla^2 \ell(w_z)) \\
 & \hspace{1.5in}
          +
   \eta\rho^2\left(
          \frac{(1+\eta \lambda_1)^3}{6}\rho
    + 2 (2\lambda_1 + B\rho)\eta
   \right) B
    \zeta,
  \end{align*}
where $\|\zeta\|\le 1$.

Thus, if we define $\epsilon:=w_t-w_z$, then for any $\rho\le c$ and $\eta\le c\rho$ for some constant $c$, there are constants $c_1$ and $c_2$ that depend on $c$, $B$ and $\lambda_1$ so that
  \[
    w_{t+1}-w_t
      = -2\epsilon + \|\epsilon\|\rho \left(c_1\nabla\lambdamax\left(\nabla^2 \ell(w_z)\right)
      +c_2\rho \zeta
      \rule[-.5\baselineskip]{0pt}{\baselineskip}
      \right).
  \]
\end{theorem}
\begin{proof}
Let 
\[
w_u = w_t + \rho\frac{\nabla\ell(w_t)}{\|\nabla\ell(w_t)\|}
\]
so that 
\[
w_{t+1} - w_t = - \eta \nabla \ell(w_u).
\]
Let 
\[
\tw_u = w_t + s_t \rho e_t = w_z + s_t (\beta_1/\lambda_1+\rho) e_1.
\]
(It may be helpful to think of $\tw_u$ as what $w_u$ would have
been, if SAM used $\ell_q$ instead of $\ell$.)
We have
\begin{align}
\label{e:in.parts}
w_{t+1} - w_t
  = - \eta \nabla \ell(\tw_u) + \eta (\nabla \ell(\tw_u) 
                                    - \nabla \ell(w_u)).
\end{align}

First, we analyze $\nabla \ell(\tw_u)$.

The fundamental theorem of calculus implies
  \begin{align*}
    &\lefteqn{D^2 \ell(w_z+\epsilon e_1)(\cdot,\cdot)} \\
      & = D^2 \ell(w_z)(\cdot,\cdot) + \int_0^1
        D^3 \ell(w_z+x\epsilon e_1)(\epsilon e_1,\cdot,\cdot)\, dx \\
      & = D^2 \ell(w_z)(\cdot,\cdot) + \int_0^1
        \left( D^3 \ell(w_z)
          + \epsilon
          \left( D^3 \ell(w_z+x\epsilon e_1) -
        D^3 \ell(w_z)\right)\right) \, dx\, (e_1,\cdot,\cdot) \\
      & = D^2 \ell(w_z)(\cdot,\cdot) +
        D^3 \ell(w_z)(\epsilon e_1,\cdot,\cdot)
          + \epsilon
          \int_0^1\left( D^3 \ell(w_z+x\epsilon e_1) -
        D^3 \ell(w_z)\right) \, dx \, (e_1,\cdot,\cdot)\\
      & = D^2 \ell(w_z)(\cdot,\cdot) +
        D^3 \ell(w_z)(\epsilon e_1,\cdot,\cdot)
          + \frac{\epsilon^2B}{2}E(\cdot,\cdot),
  \end{align*}
where the linear operator $E$ satisfies $\|E\|\le 1$. Hence (using $E$ to also denote the corresponding matrix),
  \begin{align*}
      \nabla^2\ell(w_z+\epsilon e_1)
        &= \nabla^2\ell(w_z) + \sum_{i,j}D^3\ell(w_z)(\epsilon e_1,e_i,e_j)e_i\otimes e_j
          +\frac{\epsilon^2B}{2}E\\
        &=\sum_i\lambda_i e_i\otimes e_i + D^3\ell(w_z)(\epsilon e_1,e_1,e_1)e_1\otimes e_1 \\*
        & \qquad{}
        + \sum_{i>1}D^3\ell(w_z)(\epsilon e_1,e_1,e_i)(e_1\otimes e_i+ e_i\otimes e_1) \\*
        & \qquad{}
        + \sum_{i>1,j>1}D^3\ell(w_z)(\epsilon e_1,e_i,e_j)e_i\otimes e_j
          +\frac{\epsilon^2B}{2}E.
  \end{align*}
 
Integrating from $x=0$ to $x=\epsilon$, we have
  \begin{align}
  \nonumber
      \nabla \ell(w_z+\epsilon e_1)
        &= \nabla \ell(w_z) + \int_0^\epsilon
        \nabla^2\ell(w_z+xe_1) e_1\, dx \\
  \nonumber
        &= \int_0^\epsilon \left(\sum_i\lambda_i e_i\otimes e_i + D^3\ell(w_z)(xe_1,e_1,e_1)e_1\otimes e_1 \right. \\*
  \nonumber
        & \qquad\qquad \left. {}
        + \sum_{i>1}D^3\ell(w_z)(xe_1,e_1,e_i)(e_1\otimes e_i+ e_i\otimes e_1) \right. \\*
\nonumber
        & \qquad\qquad \left. {}
        + \sum_{i>1,j>1}D^3\ell(w_z)(xe_1,e_i,e_j)e_i\otimes e_j
            + \frac{x^2B}{2}E\right)e_1\, dx \\
  \nonumber
        &= \int_0^\epsilon \left(\lambda_1 e_1 + D^3\ell(w_z)(xe_1,e_1,e_1)e_1 +
        \sum_{i>1}D^3\ell(w_z)(xe_1,e_1,e_i)e_i
          + \frac{x^2B}{2}Ee_1 \right)\, dx \\
 \nonumber
        &= \epsilon\lambda_1 e_1 + \frac{\epsilon^2}{2}
          \sum_iD^3\ell(w_z)(e_1,e_1,e_i)e_i + \frac{\epsilon^3B}{6}Ee_1.
  \end{align}
Substituting $\epsilon=s_t (\beta_1/\lambda_1+\rho)$, the first term is
  \begin{align*}
    \epsilon\lambda_1 e_1
      & = s_t \left(\frac{\beta_1}{\lambda_1}+\rho\right) \lambda_1 e_1 \\
      & = s_t \left(\frac{\eta\rho\lambda_1^2}{2-\eta\lambda_1}
        +\rho\lambda_1\right) e_1 \\
      & = \frac{2\rho\lambda_1 s_t }{2-\eta\lambda_1} e_1 \\
      & = \frac{2\beta_1 s_t }{\eta\lambda_1} e_1.
  \end{align*}
Thus,
  \begin{align}
 \nonumber
    &\lefteqn{\eta \nabla \ell(\tw_u)} \\
    \nonumber
        &= \frac{2\beta_1 s_t }{\lambda_1} e_1 + \eta
        \frac{(\beta_1/\lambda_1+\rho)^2}{2}
          \sum_iD^3\ell(w_z)(e_1,e_1,e_i)e_i +
          \eta s_t \frac{(\beta_1/\lambda_1+\rho)^3B}{6}Ee_1 \\
          \label{e:grad_near_sam_update}
        &=  \frac{2\beta_1 s_t }{\lambda_1} e_1 +
          \frac{\eta(\beta_1/\lambda_1+\rho)^2}{2}
          \nabla\lambdamax(\nabla^2 \ell(w_z)) +
          \frac{\eta(\beta_1/\lambda_1+\rho)^3 B}{6} \zeta,
  \end{align}
where $\|\zeta\|\le 1$.

Now, we turn to bounding
$
|| \nabla \ell(\tw_u) - \nabla \ell(w_u) ||.
$
(We will show that $\tw_u$ and $w_u$ are both close to $w_z$, so
that the operator norm of the Hessian is not too big between them,
and then we will show that they are close to one another.)
First, by the triangle inequality,
\[
\max \{ || \tw_u - w_z ||, || w_u - w_z || \} \leq \beta_1/\lambda_1 + \rho.
\]
Since $D^3 \ell$ is $B$-Lipschitz, this implies that,
for every $w$ on the path from $w_u$ to $\tw_u$,
\begin{align}
\label{e:opnorm}
|| \nabla^2 \ell(w) || \leq \lambda_1 + B (\beta_1/\lambda_1 + \rho).
\end{align}
Furthermore, we have
\begin{align}
\label{e:w_tw}
|| w_u - \tw_u || 
 & = \rho\left\|s_t e_1 - \frac{\nabla \ell(w_t)}{|| \nabla \ell(w_t) ||}
   \right\|.
\end{align}

Next,
\begin{align*}
\nonumber
\nabla \ell(w_t)   
& =\nabla \ell\left(w_z + \frac{s_t \beta_1}{\lambda_1} e_1 \right)   \\
& =  \nabla\ell(w_z) + \int_0^1 \nabla^2\ell\left(w_z+x\left(\frac{s_t \beta_1}{\lambda_1} e_1\right)\right)\left(\frac{s_t \beta_1}{\lambda_1}e_1\right)\, dx
            \\
\nonumber
& =  \nabla^2\ell(w_z) \left(\frac{s_t \beta_1}{\lambda_1}e_1\right) \\* 
  \nonumber
 & \qquad {}
  + \int_0^1 \left(\nabla^2\ell\left(w_z+x\left(\frac{s_t \beta_1}{\lambda_1} e_1\right)\right)-\nabla^2\ell(w_z)\right) \left(\frac{s_t \beta_1}{\lambda_1}e_1\right)\, dx \\
  \nonumber
& = s_t \beta_1 e_1+ \frac{B \beta_1^2}{2 \lambda_1^2} \xi
\end{align*}
for $\xi \in \R^d$ with $\|\xi\| \le 1$.

This implies $|| \nabla \ell(w_t) || \geq \beta_1 - \frac{B\beta_1^2}{2\lambda_1^2}$, which
in turn implies

\begin{align*}
\left\|
 \frac{\nabla\ell(w_t)}{\|\nabla\ell(w_t)\|}  
 - s_t e_1
 \right\|
  & = \left\|\left(\frac{\beta_1}{\|\nabla\ell(w_t)\|} - 1\right)s_t e_1
  + \frac{B\beta_1^2}{2\lambda_1^2\|\nabla\ell(w_t)\|}\xi\right\| \\
  & \le \frac{B\beta_1/(2\lambda_1^2)}{1-B\beta_1/(2\lambda_1^2)} +\frac{B\beta_1}{2\lambda_1^2(1-B\beta_1/(2\lambda_1^2))} \\
  & = \frac{2B\beta_1}{2\lambda_1^2-B\beta_1}.
\end{align*}

Recalling \eqref{e:w_tw},
\[
|| w_u - \tw_u || \leq \frac{2B\beta_1\rho}{2\lambda_1^2-B\beta_1},
\]
and by \eqref{e:opnorm}, this implies
\[
|| \nabla \ell(w_u) - \nabla \ell(\tw_u) || 
\le \frac{2B\beta_1\rho\left(\lambda_1 + B\beta_1/\lambda_1 + B\rho\right)}{2\lambda_1^2-B\beta_1}.
\]

Putting this together with \eqref{e:grad_near_sam_update} and \eqref{e:in.parts},
there is a $\zeta$ with $|| \zeta || \leq 1$ for which
\begin{align*}
w_{t+1} - w_t
& = -\frac{2\beta_1 s_t }{\lambda_1} e_1 -
          \frac{\eta(\beta_1/\lambda_1+\rho)^2}{2}
          \nabla\lambdamax(\nabla^2 \ell(w_z)) \\*
          & \qquad {} +
   \left(
          \frac{\eta(\beta_1/\lambda_1+\rho)^3 B}{6} 
    + \frac{2B\eta \beta_1\rho 
    \left(\lambda_1 + B\beta_1/\lambda_1 + B\rho\right)
    }{2\lambda_1^2-B\beta_1}
   \right)
    \zeta
\end{align*}
which, substituting the value of $\beta_1$ and
applying $B \eta \rho \leq 1$ and $\eta\lambda_1<1$, implies
\begin{align*}
w_{t+1} - w_t
& = 
      -2\frac{\eta \rho\lambda_1 s_t}{2-\eta\lambda_1} 
    e_1 -
          \frac{\eta}{2}\left(
            \frac{\eta \rho\lambda_1}{2-\eta\lambda_1}
              +\rho\right)^2
          \nabla\lambdamax(\nabla^2 \ell(w_z)) \\*
          & \qquad {} +
   \eta\rho^2\left(
          \frac{(1+\eta \lambda_1)^3\rho}{6}
    + 2 (2\lambda_1 + B\rho)\eta
   \right) B
    \zeta.
\end{align*}

\end{proof}

\section{Additional Simulations}

Figure~\ref{f:sam_vs_grad_start_close} compares the
trajectories of SAM (in blue) and batch gradient descent
(in green)
applied to $\frac{w_1^2}{1 + w_2^2/2} + w_2^2/2$.
It may be helpful to think of this objective as a perturbation
of the quadratic objective $w_1^2 + w_2^2/2$, that has the same
minimum, but, as $w_2$ moves away from zero, is less sharp, in the sense
that its Hessian has a smaller operator norm.  When SAM and
GD are both started $(0.1, 0.1)$, with 
$\eta = 1/5$ and $\rho = 1$, GD dives toward the minimum of
$0$, where SAM's oscillation drives it toward less sharp
solutions with larger objective values.
\begin{figure}
    \centering
    \includegraphics[width=5in]{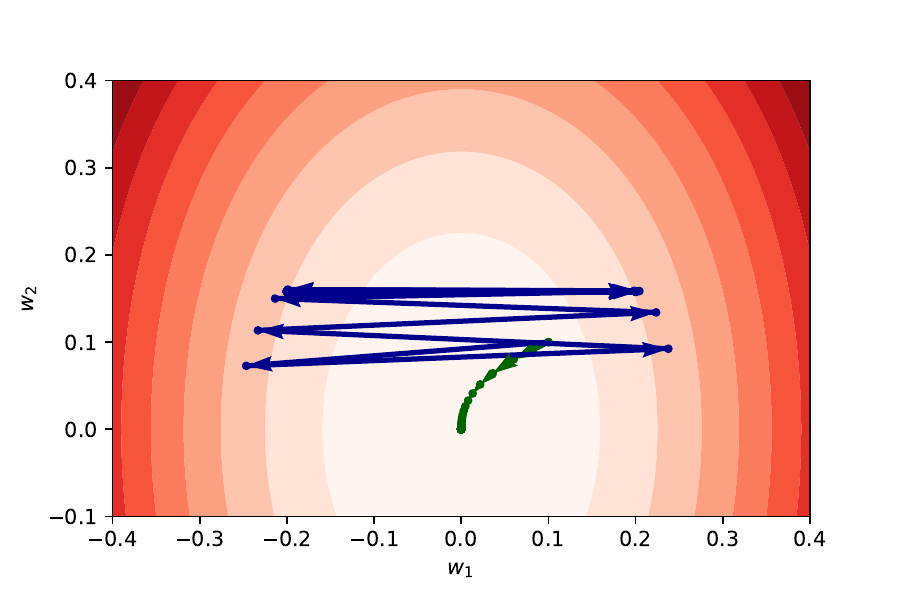}
    \caption{SAM (in blue) and gradient descent (in green) applied to $\frac{w_1^2}{1 + w_2^2/2} + w_2^2/2$ from an initial solution
    of $(0.1, 0.1)$ with $\eta = 1/5$ and $\rho = 1$.}
    \label{f:sam_vs_grad_start_close}
\end{figure}

Figure~\ref{f:sam_vs_grad_far} compares the trajectories of
SAM and GD in the same setting, except from the initial solution
$(1,1)$.  SAM behaves similarly to GD until they get close to the
origin, where SAM's oscillations carry it to a less sharp minimum
with a larger objective value.
\begin{figure}
    \centering
    \includegraphics[width=5in]{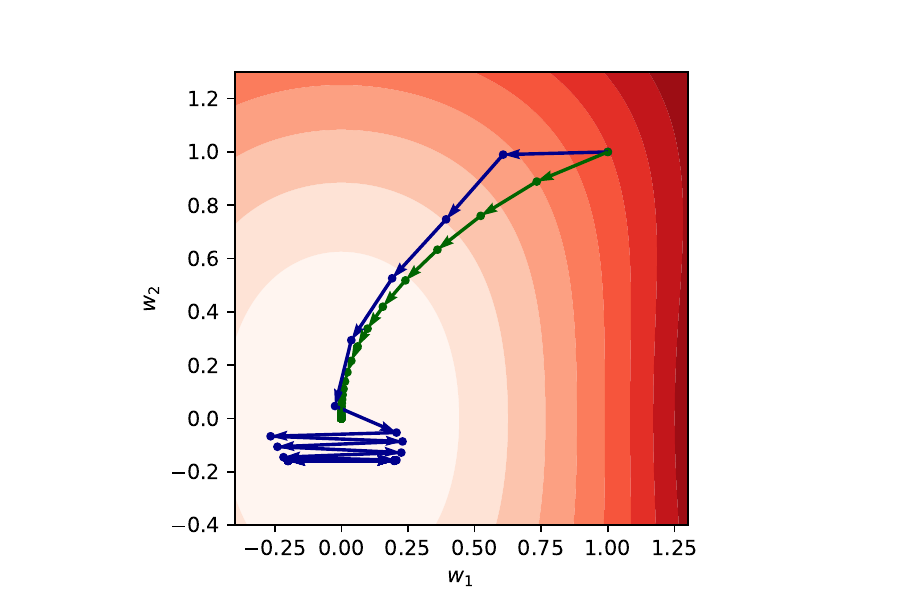}
    \caption{SAM (in blue) and gradient descent (in green) applied to $\frac{w_1^2}{1 + w_2^2/2} + w_2^2/2$ from an initial solution
    of $(1, 1)$ with $\eta = 1/5$ and $\rho = 1$.}
    \label{f:sam_vs_grad_far}
\end{figure}

Figure~\ref{f:sam_vs_sgd} compares the
trajectories of SAM and SGD, where
each stochastic gradient is obtained by
perturbing the gradient by
a sample from
$\cN(0, \sigma^2 I )$,
for $\sigma = \rho/(2 - \eta)$.  
The perturbed gradients make the iterates
of SGD sample a mix of solutions with varying
smoothness, where SAM systematically drifts
toward less sharp solutions.
\begin{figure}
    \centering
    \includegraphics[width=5in]{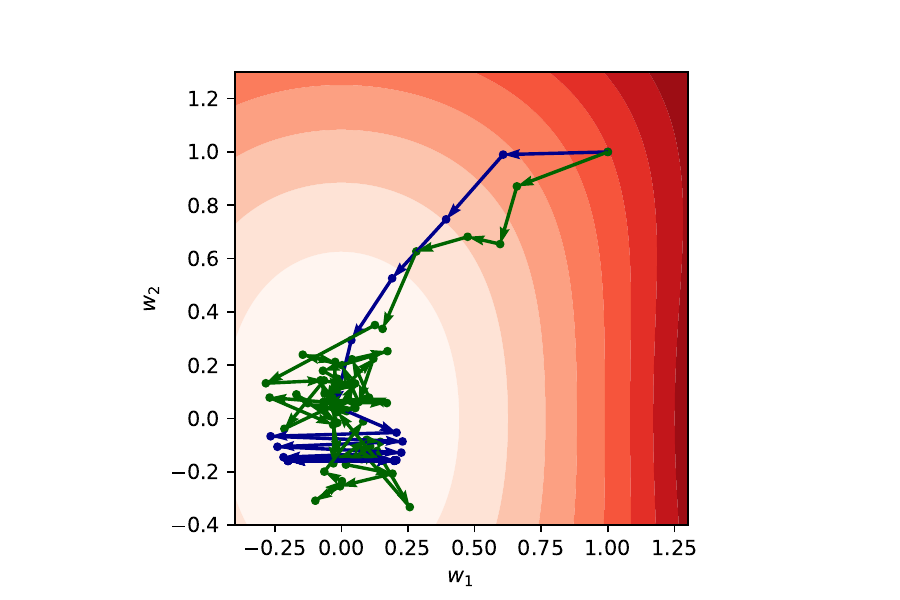}
    \caption{SAM (in blue) and SGD (in green) applied to $\frac{w_1^2}{1 + w_2^2/2} + w_2^2/2$ from an initial solution
    of $(1, 1)$ with $\eta = 1/5$, $\rho = 1$
    and $\sigma = \rho/(2 - \eta)$.}
    \label{f:sam_vs_sgd}
\end{figure}

\section{Conclusions and Open Problems}\label{s:conclusions}

Our main result, Theorem~\ref{t:medium}, shows that SAM with a
convex quadratic objective converges to a cycle that bounces across the minimum in the direction with the largest curvature. Theorem~\ref{t:drift} shows that for a locally quadratic loss, these oscillations allow gradient descent on the spectral norm of the Hessian of the loss. SAM uses one additional gradient measurement per iteration to compute a specific third derivative: the gradient of the second derivative in the leading eigenvector direction.  

Without the assumption that $\lambda_1 > \lambda_2$, Theorem~\ref{t:medium} 
would necessarily be more complex, since, informally, if $\lambda_1 = \lambda_2$, all solutions
in the span of $e_1$ and $e_2$ are equivalent.  It should not be hard to remove this
assumption while complicating some of the proofs, but without significant changes to the main ideas.

This work raises several natural questions. 
First, how is the generalization behavior affected by drifting towards wide minima? There have been several empirical studies of stochastic gradient methods for deep networks that suggest favorable generalization performance of wide minima~\citep{kmnst-olbtdlggsm-16,ccslbbcsz-esbgdiwv-16}.
There have been some analyses aimed at understanding this phenomenon based on information theoretic arguments~\citep{hvc-knns-93,hs-fm-97,NEURIPS2019_05ae14d7} and PAC-Bayes arguments~\citep{NIPS2001_98c72428,dr-cngbdsnnmmptd-17}.
It is clear that any argument about generalization properties must take account of how an algorithm solves an optimization problem over a parameterized class of functions, since wide minima are a property of a parameterization~\citep{dpbb-smcgfdn-17}.

Second, how does gradient descent on the spectral norm of the Hessian behave, particularly in the highly overparameterized setting of deep networks? When other optimization tools, such as momentum, are incorporated, how does this affect the behavior of SAM? What is the nature of SAM's solutions for losses, like the logistic loss, that are minimized at infinity?

On the technical side, it is straightforward to extend Lemma~\ref{l:u} to a local version, showing that SAM with a locally quadratic loss converges to a neighborhood of the stationary points of a function $J$ defined in terms of the Hessian. It is less straightforward to show that SAM avoids the suboptimal stationary points of $J$. It seems likely that this is true for a stochastic version of the SAM updates, and the techniques developed by~\citet{ge2015escaping,fang2019sharp} should be useful here,
which could lead to a nonasymptotic counterpart
of results of \citet{wen2022does} for a stochastic (batch-size 1) version of SAM.

Finally, can other higher derivatives be computed in the same parsimonious way as SAM? Are there related minimization methods that target other kinds of minima, for instance, by optimizing other measures of width of a minimum?


\section*{Acknowledgements}

Thanks to Hossein Mobahi for helpful comments.

\appendix

\vskip 0.2in
\bibliographystyle{plainnat}
\bibliography{bib}

\begin{thebibliography}{24}
\providecommand{\natexlab}[1]{#1}
\providecommand{\url}[1]{\texttt{#1}}
\expandafter\ifx\csname urlstyle\endcsname\relax
  \providecommand{\doi}[1]{doi: #1}\else
  \providecommand{\doi}{doi: \begingroup \urlstyle{rm}\Url}\fi

\bibitem[Ahn et~al.(2022)Ahn, Zhang, and Sra]{pmlr-v162-ahn22a}
Kwangjun Ahn, Jingzhao Zhang, and Suvrit Sra.
\newblock Understanding the unstable convergence of gradient descent.
\newblock In \emph{ICML}, pages 247--257, 2022.

\bibitem[Arora et~al.(2022)Arora, Li, and Panigrahi]{arora2022understanding}
Sanjeev Arora, Zhiyuan Li, and Abhishek Panigrahi.
\newblock Understanding gradient descent on edge of stability in deep learning.
\newblock \emph{arXiv preprint arXiv:2205.09745}, 2022.

\bibitem[Azulay et~al.(2021)Azulay, Moroshko, Nacson, Woodworth, Srebro,
  Globerson, and Soudry]{azulay2021implicit}
Shahar Azulay, Edward Moroshko, Mor~Shpigel Nacson, Blake Woodworth, Nathan
  Srebro, Amir Globerson, and Daniel Soudry.
\newblock On the implicit bias of initialization shape: Beyond infinitesimal
  mirror descent, 2021.

\bibitem[Bahri et~al.(2022)Bahri, Mobahi, and Tay]{bahri-etal-2022-sharpness}
Dara Bahri, Hossein Mobahi, and Yi~Tay.
\newblock Sharpness-aware minimization improves language model generalization.
\newblock In \emph{Proceedings of the 60th Annual Meeting of the Association
  for Computational Linguistics (Volume 1: Long Papers)}, pages 7360--7371,
  Dublin, Ireland, May 2022. Association for Computational Linguistics.
\newblock \doi{10.18653/v1/2022.acl-long.508}.
\newblock URL \url{https://aclanthology.org/2022.acl-long.508}.

\bibitem[Barrett and Dherin(2020)]{barrett2020implicit}
David~GT Barrett and Benoit Dherin.
\newblock Implicit gradient regularization.
\newblock \emph{arXiv preprint arXiv:2009.11162}, 2020.

\bibitem[Bartlett et~al.(2021)Bartlett, Montanari, and Rakhlin]{bmr-dlasp-21}
Peter~L. Bartlett, Andrea Montanari, and Alexander Rakhlin.
\newblock Deep learning: a statistical viewpoint.
\newblock \emph{Acta Numerica}, 30:\penalty0 87–201, 2021.
\newblock \doi{10.1017/S0962492921000027}.
\newblock URL \url{https://arxiv.org/abs/2103.09177}.

\bibitem[Beugnot et~al.(2022)Beugnot, Mairal, and Rudi]{beugnot2022benefits}
Gaspard Beugnot, Julien Mairal, and Alessandro Rudi.
\newblock On the benefits of large learning rates for kernel methods.
\newblock \emph{arXiv preprint arXiv:2202.13733}, 2022.

\bibitem[Chaudhari et~al.(2016)Chaudhari, Choromanska, Soatto, LeCun, Baldassi,
  Borgs, Chayes, Sagun, and Zecchina]{ccslbbcsz-esbgdiwv-16}
Pratik Chaudhari, Anna Choromanska, Stefano Soatto, Yann LeCun, Carlo Baldassi,
  Christian Borgs, Jennifer Chayes, Levent Sagun, and Riccardo Zecchina.
\newblock Entropy-sgd: Biasing gradient descent into wide valleys.
\newblock \emph{arXiv:1611.01838}, 2016.

\bibitem[Cohen et~al.(2020)Cohen, Kaur, Li, Kolter, and
  Talwalkar]{cohen2020gradient}
Jeremy Cohen, Simran Kaur, Yuanzhi Li, J~Zico Kolter, and Ameet Talwalkar.
\newblock Gradient descent on neural networks typically occurs at the edge of
  stability.
\newblock In \emph{International Conference on Learning Representations}, 2020.

\bibitem[Damian et~al.(2022)Damian, Nichani, and Lee]{damian2022self}
Alex Damian, Eshaan Nichani, and Jason~D Lee.
\newblock Self-stabilization: The implicit bias of gradient descent at the edge
  of stability.
\newblock \emph{arXiv preprint arXiv:2209.15594}, 2022.

\bibitem[Dinh et~al.(2017)Dinh, Pascanu, Bengio, and Bengio]{dpbb-smcgfdn-17}
Laurent Dinh, Razvan Pascanu, Samy Bengio, and Yoshua Bengio.
\newblock Sharp minima can generalize for deep nets.
\newblock \emph{arXiv:1703.04933}, 2017.

\bibitem[Du et~al.(2022)Du, Yan, Feng, Zhou, Zhen, Goh, and
  Tan]{du2022efficient}
Jiawei Du, Hanshu Yan, Jiashi Feng, Joey~Tianyi Zhou, Liangli Zhen, Rick
  Siow~Mong Goh, and Vincent Tan.
\newblock Efficient sharpness-aware minimization for improved training of
  neural networks.
\newblock In \emph{International Conference on Learning Representations}, 2022.

\bibitem[Dziugaite and Roy(2017)]{dr-cngbdsnnmmptd-17}
Gintare~Karolina Dziugaite and Daniel~M. Roy.
\newblock Computing nonvacuous generalization bounds for deep (stochastic)
  neural networks with many more parameters than training data.
\newblock \emph{arXiv:1703.11008}, 2017.

\bibitem[Fang et~al.(2019)Fang, Lin, and Zhang]{fang2019sharp}
Cong Fang, Zhouchen Lin, and Tong Zhang.
\newblock Sharp analysis for nonconvex sgd escaping from saddle points.
\newblock In \emph{Conference on Learning Theory}, pages 1192--1234, 2019.

\bibitem[Foret et~al.(2021)Foret, Kleiner, Mobahi, and
  Neyshabur]{foret2021sharpnessaware}
Pierre Foret, Ariel Kleiner, Hossein Mobahi, and Behnam Neyshabur.
\newblock Sharpness-aware minimization for efficiently improving
  generalization.
\newblock \emph{arXiv:2010.01412}, 2021.

\bibitem[Ge et~al.(2015)Ge, Huang, Jin, and Yuan]{ge2015escaping}
Rong Ge, Furong Huang, Chi Jin, and Yang Yuan.
\newblock Escaping from saddle points—online stochastic gradient for tensor
  decomposition.
\newblock In \emph{Conference on learning theory}, pages 797--842. PMLR, 2015.

\bibitem[Hinton and van Camp(1993)]{hvc-knns-93}
Geoffrey~E. Hinton and Drew van Camp.
\newblock Keeping the neural networks simple by minimizing the description
  length of the weights.
\newblock In \emph{Proceedings of the Sixth Annual Conference on Computational
  Learning Theory}, COLT '93, page 5–13, 1993.
\newblock \doi{10.1145/168304.168306}.
\newblock URL \url{https://doi.org/10.1145/168304.168306}.

\bibitem[Hochreiter and Schmidhuber(1997)]{hs-fm-97}
Sepp Hochreiter and Jürgen Schmidhuber.
\newblock {Flat Minima}.
\newblock \emph{Neural Computation}, 9\penalty0 (1):\penalty0 1--42, 1997.
\newblock \doi{10.1162/neco.1997.9.1.1}.

\bibitem[Keskar et~al.(2016)Keskar, Mudigere, Nocedal, Smelyanskiy, and
  Tang]{kmnst-olbtdlggsm-16}
Nitish~Shirish Keskar, Dheevatsa Mudigere, Jorge Nocedal, Mikhail Smelyanskiy,
  and Ping Tak~Peter Tang.
\newblock On large-batch training for deep learning: Generalization gap and
  sharp minima.
\newblock \emph{arXiv:1609.04836}, 2016.

\bibitem[Langford and Caruana(2001)]{NIPS2001_98c72428}
John Langford and Rich Caruana.
\newblock (not) bounding the true error.
\newblock In T.~Dietterich, S.~Becker, and Z.~Ghahramani, editors,
  \emph{Advances in Neural Information Processing Systems}, volume~14. MIT
  Press, 2001.
\newblock URL
  \url{https://proceedings.neurips.cc/paper/2001/file/98c7242894844ecd6ec94af67ac8247d-Paper.pdf}.

\bibitem[Negrea et~al.(2019)Negrea, Haghifam, Dziugaite, Khisti, and
  Roy]{NEURIPS2019_05ae14d7}
Jeffrey Negrea, Mahdi Haghifam, Gintare~Karolina Dziugaite, Ashish Khisti, and
  Daniel~M Roy.
\newblock Information-theoretic generalization bounds for sgld via
  data-dependent estimates.
\newblock In H.~Wallach, H.~Larochelle, A.~Beygelzimer, F.~d\textquotesingle
  Alch\'{e}-Buc, E.~Fox, and R.~Garnett, editors, \emph{Advances in Neural
  Information Processing Systems}, volume~32. Curran Associates, Inc., 2019.
\newblock URL
  \url{https://proceedings.neurips.cc/paper/2019/file/05ae14d7ae387b93370d142d82220f1b-Paper.pdf}.

\bibitem[Smith et~al.(2021)Smith, Dherin, Barrett, and De]{smith2021origin}
Samuel~L Smith, Benoit Dherin, David~GT Barrett, and Soham De.
\newblock On the origin of implicit regularization in stochastic gradient
  descent.
\newblock \emph{arXiv preprint arXiv:2101.12176}, 2021.

\bibitem[Soudry et~al.(2018)Soudry, Hoffer, Nacson, Gunasekar, and
  Srebro]{soudry2018implicit}
Daniel Soudry, Elad Hoffer, Mor~Shpigel Nacson, Suriya Gunasekar, and Nathan
  Srebro.
\newblock The implicit bias of gradient descent on separable data, 2018.

\bibitem[Wen et~al.(2022)Wen, Ma, and Li]{wen2022does}
Kaiyue Wen, Tengyu Ma, and Zhiyuan Li.
\newblock How does sharpness-aware minimization minimize sharpness?
\newblock \emph{arXiv preprint arXiv:2211.05729}, 2022.

\end{thebibliography}

\end{document}